\documentclass{article}
\usepackage{ijcai23}

\usepackage{times}
\usepackage{soul}
\usepackage{url}
\usepackage[hidelinks]{hyperref}
\usepackage[utf8]{inputenc}
\usepackage[small]{caption}
\usepackage{graphicx}
\usepackage{amsmath}
\usepackage{amsthm}
\usepackage{enumitem}
\usepackage{booktabs}
\usepackage{mathtools}
\usepackage{natbib}

\usepackage{caption}
\usepackage{subcaption}

\usepackage{setspace}
\usepackage{multirow}
\usepackage{bm,bbm}
\usepackage{enumitem}
\usepackage{amsmath}
\usepackage{amssymb}
\usepackage{amsthm}
\usepackage{mathtools}
\usepackage{amsfonts}
\usepackage{booktabs}       % professional-quality tables
\usepackage{nicefrac}  
\usepackage{amsthm}% compact symbols for 1/2, etc.

\usepackage[switch]{lineno}
\usepackage[ruled,noresetcount,vlined,linesnumbered]{algorithm2e}
\usepackage{xcolor}

\DeclareMathOperator*{\argmax}{argmax}
\DeclareMathOperator*{\argmin}{argmin}

\newtheorem{theorem}{Theorem}
\newtheorem{definition}{Definition}
\newtheorem{corollary}{Corollary}
\newtheorem{lemma}{Lemma}

\title{
Backpropagation of Unrolled Solvers with Folded Optimization
}

\author{James Kotary  
\and My H Dinh 
\And Ferdinando Fioretto \\
\affiliations
  University of Virginia
\emails
\{jk4pn, fqw2tz, fioretto\}@virginia.edu
}

\begin{document}
\maketitle

\begin{abstract}
The integration of constrained optimization models as components in deep networks has led to promising advances on many specialized learning tasks. 
A central challenge in this setting is backpropagation through the solution of an optimization problem, which typically lacks a closed form. One typical strategy is algorithm unrolling, which relies on automatic differentiation through the operations of an iterative solver. While flexible and general, unrolling can encounter accuracy and efficiency issues in practice. These issues can be avoided by analytical differentiation of the optimization, but current frameworks impose rigid requirements on the optimization problem's form. This paper provides theoretical insights into the backward pass of unrolled optimization, leading to a system for generating  efficiently solvable analytical models of backpropagation. Additionally, it proposes a unifying view of unrolling and analytical differentiation through optimization mappings. Experiments over various model-based learning tasks demonstrate the advantages of the approach both computationally and in terms of enhanced expressiveness.
%\footnote{\textbf{Supplemental material:} Please refer to \citep{kotary2023folded} for an extended version of this paper including Appendix additional related work, implementation details, analysis, and experiments. All references to the Appendix refer to that of the extended version.}.
\end{abstract}

\maketitle\sloppy\allowdisplaybreaks

%%%%%%%%%%%%%%%%%%%%%%%%%%%%%%%%%%%%%%%%%%%%%
\section{Introduction}
\label{sec:introduction}
%%%%%%%%%%%%%%%%%%%%%%%%%%%%%%%%%%%%%%%%%%%%%

The integration of optimization problems as components in neural networks has shown to be an effective framework for enforcing structured representations in deep learning. A parametric optimization problem defines a mapping from its unspecified parameters to the resulting optimal solutions, which is treated as a layer of a neural network. Outputs of the layer are then guaranteed to obey the problem's constraints, which may be predefined or learned \citep{kotary2021end}. 

Using optimization as a layer can offer enhanced accuracy and efficiency on specialized learning tasks by imparting task-specific structural knowledge. For example, it has been used to design efficient multi-label classifiers and sparse attention mechanisms \citep{martins2016softmax}, learning to rank based on optimal matching \citep{adams2011ranking, kotary2022end}, 
 accurate model selection protocols \citep{kotary2023_IJCAI}, and enhanced models for optimal decision-making under uncertainty \citep{wilder2019melding}. 

While constrained optimization mappings can be used as components in neural networks in a similar manner to linear layers or activation functions \citep{amos2019optnet}, a prerequisite is their differentiation, for backpropagation of gradients in end-to-end training by stochastic gradient descent.

This poses unique challenges, partly due to their lack of a closed form, and modern approaches typically follow one of two strategies. In \emph{unrolling}, an optimization algorithm is executed entirely on the computational graph, and backpropagated by automatic differentiation from optimal solutions to the underlying problem parameters. The approach is adaptable to many problem classes, but has been shown to suffer from time and space inefficiency, as well as vanishing gradients \citep{monga2021algorithm}. \emph{Analytical differentiation} is a second strategy that circumvents those issues by forming implicit models for the derivatives of an optimization mapping and solving them exactly. However, current frameworks have rigid requirements on the form of the optimization problems, such as relying on transformations to canonical convex cone programs before applying a standardized procedure for their solution and differentiation  \citep{agrawal2019differentiable}.  This  system precludes the use of specialized solvers that are best-suited to handle various optimization problems, and inherently restricts itself only to convex problems.\footnote{A discussion of related work on differentiable optimization and decision-focused learning is provided in Appendix 
A.}

\paragraph{Contributions.}
To address these limitations, this paper proposes a novel analysis of unrolled optimization, which results in efficiently-solvable models for the backpropagation of unrolled optimization. Theoretically, the result is significant because it establishes an equivalence between unrolling and analytical differentiation, and allows for convergence of the backward pass to be analyzed in unrolling. Practically, it allows for the forward and backward passes of unrolled optimization to be disentangled and solved separately, using blackbox implementations of specialized algorithms. More specifically, this paper makes the following novel contributions:      
\textbf{(1)} A theoretical analysis of unrolling that leads to an efficiently solvable closed-form model, whose solution is equivalent to the backward pass of an unrolled optimizer.  
\textbf{(2)} Building on this analysis, it proposes a system for generating  analytically differentiable optimizers from unrolled implementations, accompanied by a Python library called \texttt{fold-opt} to facilitate automation. \textbf{(3)} Its efficiency and modeling advantages are demonstrated on a diverse set of end-to-end optimization and learning tasks, including the first demonstration of decision-focused learning with \emph{nonconvex} decision models.
    
%, available at \href{https://github.com/AIPOpt-Lab-SU/folded_optimization}{https://github.com/AIPOpt-Lab-SU/folded\_optimization}.

\section{Setting and Goals}
\label{sec:SettingAndGoals}

In this paper, the goal is to differentiate mappings that are defined as the solution to an optimization problem.  Consider the parameterized problem  (\ref{eq:opt_generic}) which defines a function from a vector of parameters $\mathbf{c} \in \mathbb{R}^p$ to its associated optimal solution $\mathbf{x}^{\star}( \mathbf{c} ) \in \mathbb{R}^n$:     
\begin{subequations} %\tag{OPT}
    \label{eq:opt_generic}
    \begin{align}
        \mathbf{x}^{\star}(\mathbf{c}) =  \argmin_{\mathbf{x}} &\; f(\mathbf{x},\mathbf{c})  \label{eq:opt_generic_objective}\\        
        \text{subject to: }&\;
        g(\mathbf{x},\mathbf{c}) \leq \mathbf{0},  \label{eq:opt_generic_inequality}\\
        &\; h(\mathbf{x},\mathbf{c}) = \mathbf{0},\label{eq:opt_generic_equality}
    \end{align}
\end{subequations}
in which $f$ is the objective function, and $g$ and $h$ are vector-valued functions capturing the inequality and equality constraints of the problem, respectively. The parameters $\mathbf{c}$ can be thought of as a prediction from previous layers of a neural network, or as learnable parameters analogous to the weights of a linear layer, or as some combination of both.   It is assumed throughout that for any $\mathbf{c}$, the associated optimal solution $\mathbf{x}^{\star}(\mathbf{c})$ can be found by conventional methods, within some tolerance in solver error. 
This coincides with the ``forward pass'' of the mapping in a neural network. \emph{The primary challenge is to compute its backward pass}, which amounts to finding the Jacobian matrix of partial derivatives $\frac{\partial \mathbf{x}^{\star}(\mathbf{c})}{\partial \mathbf{c}}$ .

\paragraph{Backpropagation.}  
 Given a downstream task loss $\mathcal{L}$, backpropagation through   $\mathbf{x}^{\star}(\mathbf{c})$ amounts to computing $\frac{\partial \mathcal{L}}{\partial \mathbf{c}}$ given $\frac{\partial \mathcal{L}}{\partial \mathbf{x}^{\star}}$.  In deep learning, backpropagation through a layer is typically accomplished by automatic differentiation (AD), which propagates gradients through the low-level operations of an overall composite function by repeatedly applying the multivariate chain rule. This can be performed automatically given a forward pass implementation in an AD library such as PyTorch. However, it requires a record of all the operations performed during the forward pass and their dependencies, known as the \emph{computational graph}.

\paragraph{Jacobian-gradient product (JgP).}
The \emph{Jacobian} matrix of the vector-valued function $\mathbf{x}^{\star}(\mathbf{c}): \mathbb{R}^p \rightarrow \mathbb{R}^n$ is a matrix $\frac{\partial \mathbf{x}^{\star}}{\partial \mathbf{c}}$ in $\mathbb{R}^{n \times p}$, whose elements at $(i,j)$  are the partial derivatives $\frac{\partial x^{\star}_i(\mathbf{c})}{\partial c_j}$. When the Jacobian is known, backpropagation through $\mathbf{x}^{\star}(\mathbf{c})$  can be performed  by computing the product
\begin{equation}
\label{eq:jvprod}
  \frac{\partial \mathcal{L}}{\partial \mathbf{c}} = \frac{\partial \mathcal{L}}{\partial \mathbf{x}^{\star}} \cdot \frac{\partial \mathbf{x}^{\star}(\mathbf{c})}{\partial \mathbf{c}}.
 \end{equation}

\begin{figure}[t]
    \centering
    \includegraphics[width=1.0\linewidth]{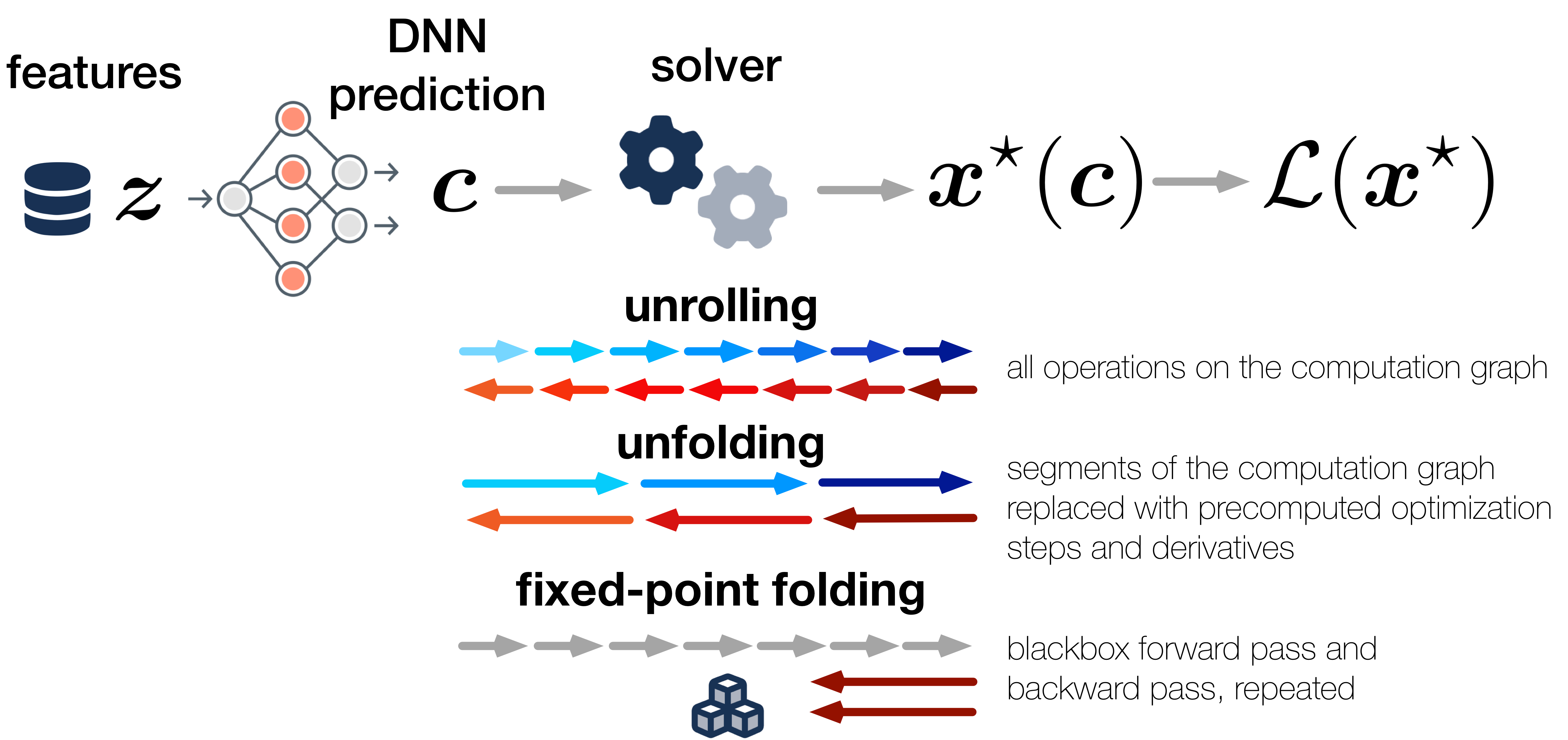}
    \caption{Compared to unrolling, unfolding requires fewer operations on the computational graph by replacing inner loops with Jacobian-gradient products.   Fixed-point folding models the unfolding analytically, allowing for blackbox implementations.}
    \label{fig:unfolding_scheme}
\end{figure}

\paragraph{Folded Optimization: Overview.}

The problem \eqref{eq:opt_generic} is most often solved by iterative methods, which refine an initial \emph{starting point} $\mathbf{x}_0$ by repeated application of a subroutine, which we view as a function. For optimization variables $\mathbf{x} \in \mathbb{R}^n$, the \emph{update function} is  a vector-valued function $\mathcal{U}: \mathbb{R}^n \rightarrow \mathbb{R}^n$:
\begin{equation}
    \label{eq:opt_iteration}\tag{U}
        \mathbf{x}_{k+1}(\mathbf{c}) = \mathcal{U}(\mathbf{x}_k (\mathbf{c}) ,\;  \mathbf{c} ).
\end{equation}
\noindent 
The iterations (\ref{eq:opt_iteration}) \emph{converge} if $\mathbf{x}_{k}(\mathbf{c}) \rightarrow \mathbf{x}^{\star}(\mathbf{c})$ as $k \rightarrow \infty$. 
When \emph{unrolling}, the iterations \eqref{eq:opt_iteration} are computed and recorded on the computational graph, and the function $\mathbf{x}^{\star}(\mathbf{c})$ can be thereby be backpropagated by AD without explicitly representing its Jacobian. However, unrolling over many iterations often faces time and space inefficiency issues due to the need for graph storage and traversal \cite{monga2021algorithm}.
%unrolling over many iterations is prone to several inefficiencies, including the time and space required to store and traverse the graph \cite{monga2021algorithm}. 
In the following sections, we show how the backward pass of unrolling can be analyzed to yield equivalent analytical models for the Jacobian of $\mathbf{x}^{\star}(\mathbf{c})$. We recognize two key challenges in modeling the backward pass of unrolling iterations \eqref{eq:opt_iteration}. 

First, it often happens that evaluation of $\mathcal{U}$ in \eqref{eq:opt_iteration} requires the solution of another optimization subproblem, such as a projection or proximal operator, which must also be unrolled. Section \ref{sec:unfolding} introduces \emph{\textbf{unfolding}} as a variant of unrolling, in which the unrolling of such inner loops is circumvented by analytical differentiation of the subproblem, allowing the analysis to be confined to a single unrolled loop. 

Second, the backward pass of an unrolled solver is determined by its forward pass, whose trajectory depends on its (potentially arbitrary) starting point and the convergence properties of the chosen algorithm. Section \ref{sec:Unfolding_at_a_fixed_point} shows that the backward pass converges correctly even when the forward pass iterations are initialized at a precomputed optimal solution. This allows for separation of the forward and backward passes, which are typically entangled across unrolled iterations, greatly simplifying the backward pass model and allowing for  blackbox implementations of both passes.

Section \ref{sec:FoldedOptimization} uses these concepts to show that the backward pass of unfolding \eqref{eq:opt_iteration} follows exactly the solution of the linear system for $\frac{\partial \mathbf{x}^{\star}(\mathbf{c})}{\partial \mathbf{c}}$ which arises by differentiating the fixed-point conditions of \eqref{eq:opt_iteration}. Section \ref{sec:Experiments} then outlines \emph{\textbf{fixed-point folding}}, a system for generating Jacobian-gradient products through optimization mappings from their unrolled solver implementations, based on efficient solution of the models proposed in Section \ref{sec:FoldedOptimization}. The main differences between unrolling, unfolding, and fixed-point folding are illustrated in Figure \ref{fig:unfolding_scheme}.

\section{From Unrolling to Unfolding}
\label{sec:unfolding}
For many optimization algorithms of the form (\ref{eq:opt_iteration}), the update function $\mathcal{U}$ is composed of closed-form functions that are relatively simple to evaluate and differentiate. In general though, $\mathcal{U}$ may itself employ an optimization subproblem that is nontrivial to differentiate. That is, 
\begin{equation}
    \label{eq:inner_optimization}\tag{O}
        \mathcal{U} ( \mathbf{x}_k   ) \coloneqq {\mathcal{T}} \left( \;  \mathcal{O}(\mathcal{S}(\mathbf{x}_k)), \;\;  \mathbf{x}_k  \;\right),
\end{equation}
\noindent 
wherein the differentiation of $\mathcal{U}$ is complicated by an \emph{inner optimization} sub-routine $\mathcal{O}: \mathbb{R}^n \rightarrow \mathbb{R}^n$. 
Here, $\mathcal{S}$ and $\mathcal{T}$ represent any steps preceding or following the inner optimization (such as gradient steps), viewed as closed-form functions. In such cases, unrolling (\ref{eq:opt_iteration}) would also require unrolling  $\mathcal{O}$. If the Jacobians of $\mathcal{O}$ can be found, then backpropagation through $\mathcal{U}$ can be completed, free of unrolling, by applying a chain rule through Equation (\ref{eq:inner_optimization}), which in this framework is handled naturally by automatic differentiation of $\mathcal{T}$ and $\mathcal{S}$.

Then, only the outermost iterations (\ref{eq:opt_iteration}) need be unrolled on the computational graph for backpropagation. This partial unrolling, which allows for backpropagating large segments of computation at a time by leveraging analytically differentiated subroutines, is henceforth referred to as \emph{unfolding}. It is made possible when the update step $\mathcal{U}$ is easier to differentiate than the overall optimization mapping $\mathbf{x}^{\star}(\mathbf{c})$.
\begin{definition}[Unfolding]
\label{def:unfolding}
An \emph{unfolded} optimization of the form (\ref{eq:opt_iteration}) is one in which the backpropagation of $\mathcal{U}$ at each step does not require unrolling an iterative algorithm.
\end{definition}
Unfolding is distinguished from more general unrolling by the presence of only a single unrolled loop. This definition sets the stage for Section \ref{sec:FoldedOptimization}, which shows how the backpropagation of an unrolled loop can be modeled with a Jacobian-gradient product. Thus, unfolded optimization is a precursor to the complete replacement of backpropagation through loops in unrolled solver implementations by JgP. 

When $\mathcal{O}$ has a closed form and does not require an iterative solution, the definitions unrolling and unfolding coincide. When $\mathcal{O}$ is nontrivial to solve but has known Jacobians, they can be used to produce an unfolding of (\ref{eq:opt_iteration}). Such is the case when $\mathcal{O}$ is a Quadratic Program (QP); a JgP-based differentiable QP solver called \texttt{qpth} is provided by  \cite{amos2019optnet}. Alternatively, the replacement of unrolled loops by JgP's proposed in Section \ref{sec:FoldedOptimization} can be applied recursively $\mathcal{O}$. 

\begin{figure}[t]
    \centering
    \includegraphics[width=0.70\linewidth]{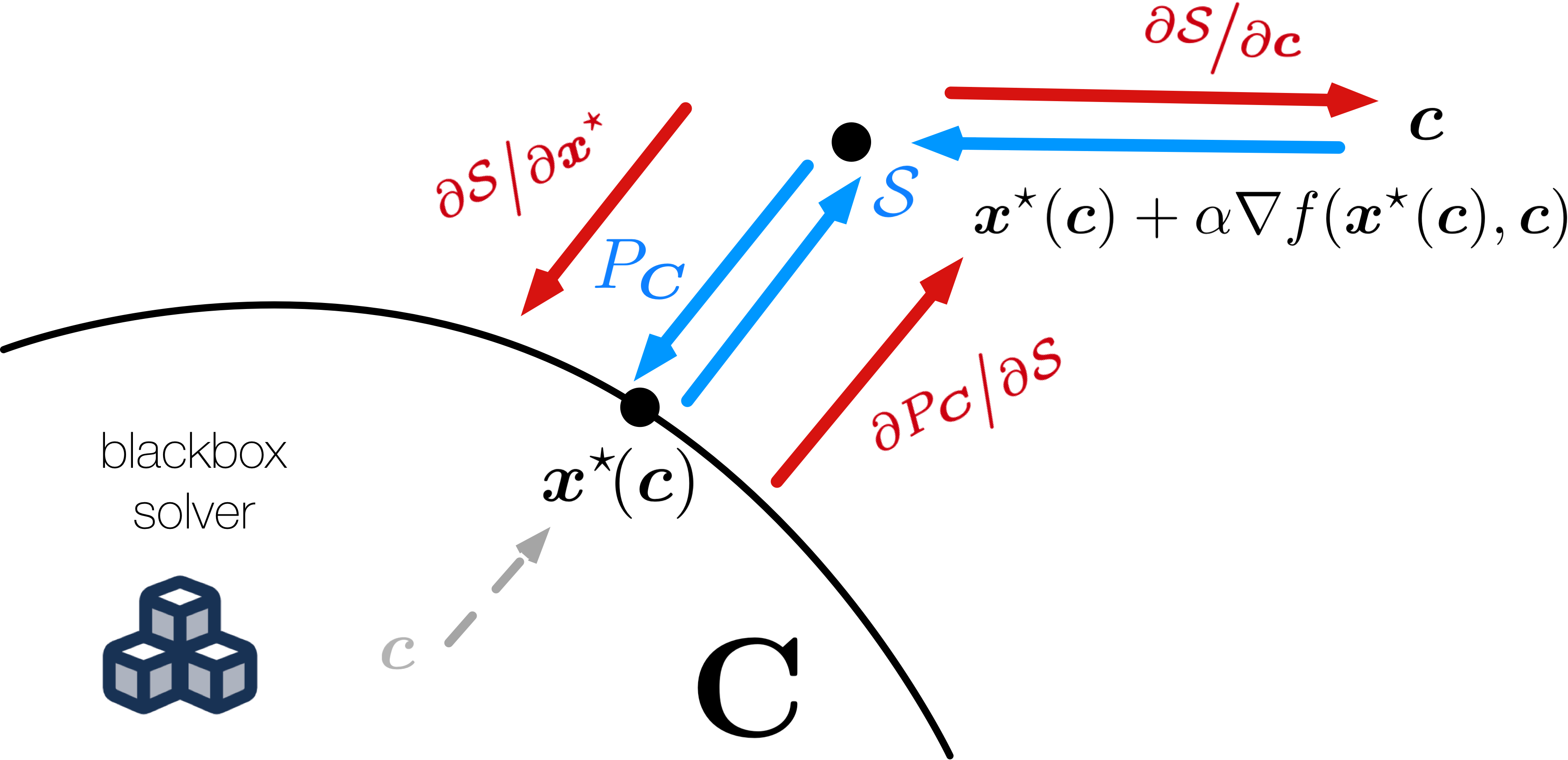}
    \caption{Unfolding Projected Gradient Descent at $\mathbf{x}^{\star}$ consists of alternating gradient step $\mathcal{S}$ with projection $\mathcal{P}_{\mathbf{C}}$. Each function's forward and backward pass are in blue and red, respectively. }
    \label{fig:unfold_pgd}
\end{figure}

These concepts are illustrated in the following examples, highlighting the roles of $\mathcal{U}$, $\mathcal{O}$  and $\mathcal{S}$. Each will be used to create folded optimization mappings for a variety of learning tasks in Section \ref{sec:Experiments}.

\paragraph{Projected gradient descent.}
Given a problem 
\begin{equation}
\label{eq:problem_pgd}
        \min_{\mathbf{x} \in \mathbf{C}} \; f(\mathbf{x})   
\end{equation}
where $f$ is differentiable and $\mathbf{C}$ is the feasible set, Projected Gradient Descent (PGD) follows the update function
\begin{equation}
\label{eq:update-pgd}
        \mathbf{x}_{k+1} =   \mathcal{P}_{\mathbf{C}}( \mathbf{x}_k - \alpha_k \nabla f (\mathbf{x}_k) ),
\end{equation}
where $\mathcal{O} = \mathcal{P}_{\mathbf{C}}$ is the Euclidean projection onto $\mathbf{C}$, and $\mathcal{S}(\mathbf{x}) = \mathbf{x} - \alpha \nabla f (\mathbf{x}) $ is a gradient descent step. Many simple $\mathbf{C}$ have closed-form projections to facilitate unfolding of (\ref{eq:update-pgd}) (see \citep{beck2017first}). Further, when $\mathbf{C}$ is linear, $\mathcal{P}_\mathbf{C}$ is a quadratic programming (QP) problem for which a differentiable solver \texttt{qpth} is available from \cite{amos2019optnet}.

Figure \ref{fig:unfold_pgd} shows one iteration of unfolding projected gradient descent, with  the forward and backward pass of each recorded operation on the computational graph illustrated in blue and red, respectively. 

\paragraph{Proximal gradient descent.}
%\label{ex:prox}
More generally, to solve 
\begin{equation}
\label{eq:problem_prox}
        \min_{\mathbf{x}} \; f(\mathbf{x}) + g(\mathbf{x}) 
\end{equation}
where $f$ is differentiable and $g$ is a closed convex function, proximal gradient descent follows the update function
\begin{equation}
\label{eq:update-prox}
        \mathbf{x}_{k+1} =   \operatorname{Prox}_{\alpha_k g}
        \left(\mathbf{x}_k - \alpha_k \nabla f (\mathbf{x}_k) \right).
\end{equation}
Here $\mathcal{O}$ is the proximal operator, defined as
\begin{equation}
\label{eq:def-prox}
\operatorname{Prox}_{g}(\mathbf{x}) = \argmin_{y} 
\left\{g(\mathbf{y}) + \frac{1}{2} 
\| \mathbf{y} - \mathbf{x}  \|^2 \right\},
\end{equation}
and its difficulty depends on $g$. Many simple proximal operators can be represented in closed form and have simple derivatives. 
For example, 
% for the indicator function $g = \delta_{\mathbb{C}}$, we have $\operatorname{Prox}_{g} = \mathcal{P}_{\mathbb{C}}$ which reduces to the previous example. In the common case 
when $g(\mathbf{x}) = \lambda \|\mathbf{x}\|_1$, then $\operatorname{Prox}_{g} = \mathcal{T}_{\lambda}(\mathbf{x})$ is the soft thresholding operator, whose closed-form formula and derivative are given in Appendix \ref{appendix:models}.

\paragraph{Sequential quadratic programming.}
Sequential Quadratic Programming (SQP)  solves the general optimization problem (\ref{eq:opt_generic}) by approximating it at each step by a QP problem, whose objective is a second-order approximation of the problem's Lagrangian function, subject to a linearization of its constraints. SQP is well-suited for  unfolded optimization, as it can solve a broad class of convex and nonconvex problems and can readily be unfolded by implementing its QP step (shown in Appendix \ref{appendix:models}) 
with the  \texttt{qpth} differentiable QP solver.

\paragraph{Quadratic programming by ADMM.}
The  QP solver of \cite{boyd2011distributed}, based on the alternating direction of multipliers, is specified in Appendix \ref{appendix:models}. 
Its inner optimization step $\mathcal{O}$ is a simpler equality-constrained QP; its solution is equivalent to solving a linear system of equations, which has a simple derivative rule in PyTorch.

\smallskip
Given an unfolded QP solver by ADMM, its unrolled loop can be replaced with backpropagation by JgP as shown in Section \ref{sec:FoldedOptimization}. The resulting  differentiable QP solver can then take the place of \texttt{qpth} in the examples above. Subsequently, \emph{this technique can be applied recursively} to the resulting unfolded PGD and SQP solvers. This exemplifies the intermediate role of unfolding in converting unrolled, nested solvers to fully JgP-based implementations, detailed in Section~\ref{sec:Experiments}.

From the viewpoint of unfolding, the analysis of backpropagation in unrolled solvers can be simplified by accounting for only a single unrolled loop. The next section identifies a further simplification: \emph{that the backpropagation of an unfolded solver can be completely characterized by its action at a fixed point of the solution's algorithm.} 

\section{Unfolding at a Fixed Point}
\label{sec:Unfolding_at_a_fixed_point}
Optimization procedures of the form  (\ref{eq:opt_iteration}) generally require a  starting point $\mathbf{x}_0$, which is often chosen arbitrarily, since convergence $\mathbf{x}_k \to \mathbf{x}^{\star}$ of iterative algorithms is typically guaranteed regardless of starting point. It is natural to then ask how the choice of $\mathbf{x}_0$ affects the convergence of the backward pass. We define backward-pass convergence as follows: 
\begin{definition}
Suppose that an unfolded iteration (\ref{eq:opt_iteration}) produces a convergent sequence of solution iterates $\lim_{k \to \infty} \mathbf{x}_k  = \mathbf{x}^{\star}$ in its forward pass. Then convergence of the backward pass is 
\begin{equation}
\label{eq:limit-x-diff}
\lim_{k \to \infty} \frac{\partial \mathbf{x}_k}{\partial \mathbf{c}}(\mathbf{c})  = \frac{\partial \mathbf{x}^{\star}}{\partial \mathbf{c}}(\mathbf{c}).
\end{equation}
\end{definition}

\begin{figure}[t]
\centering
    \vspace{-12pt}
    \includegraphics[width=0.8\linewidth]{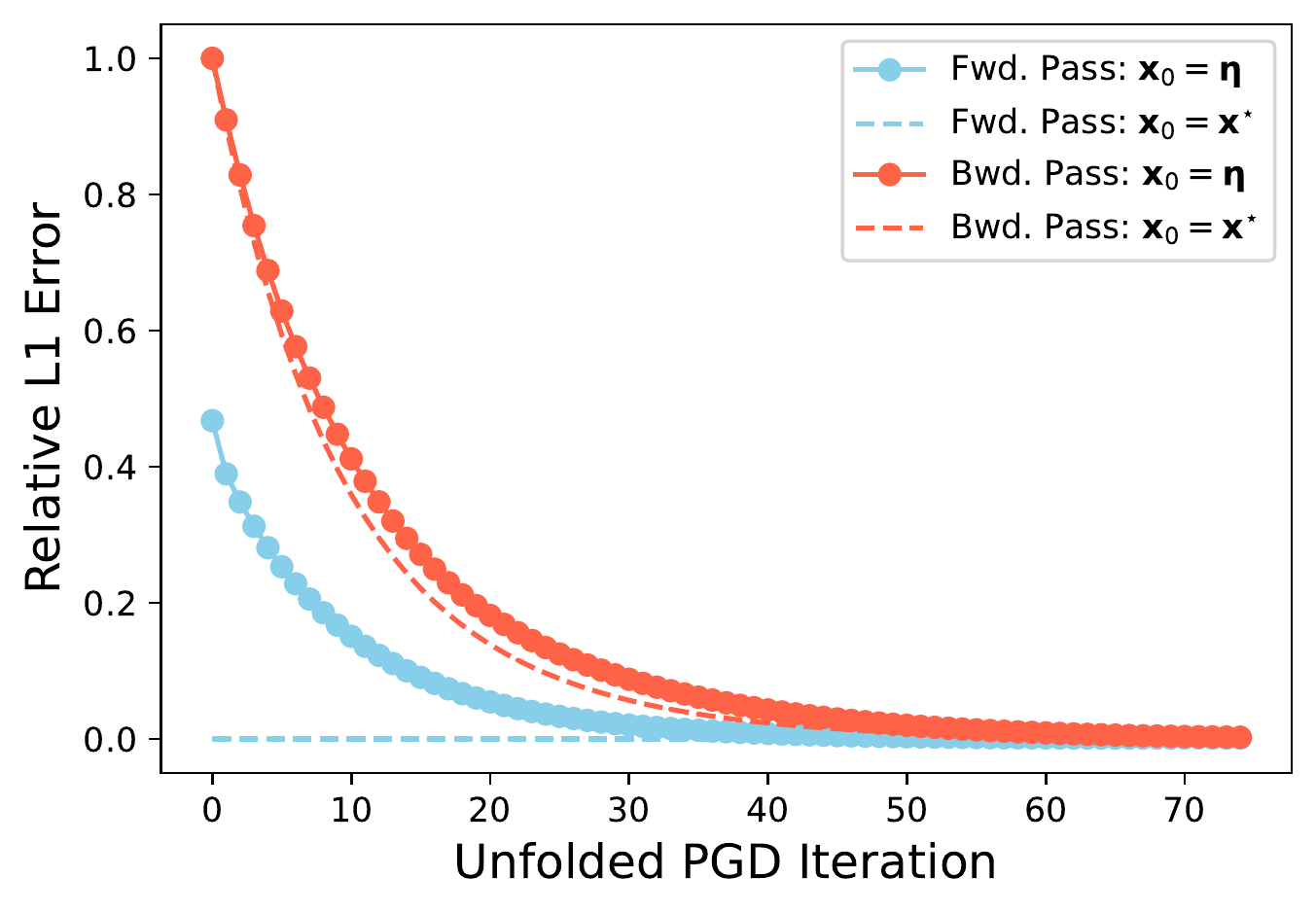} 
    \caption{Forward and backward pass error in unfolding PGD}
    \label{fig:fwd_bwd_err}
\end{figure}

\paragraph{Effect of the starting point on backpropagation.} 
Consider the optimization mapping (\ref{eq:topk-lp}) which maps feature embeddings to smooth top-$k$ class predictions, and will be  used to learn multilabel classification later in Section \ref{sec:Experiments}. 
A loss function $\mathcal{L}$ targets ground-truth top-$k$ indicators, and the result of the backward pass is the gradient $\frac{\partial \mathcal{L}}{\partial \mathbf{c}}$.  
To evaluate backward pass convergence in unfolded projected gradient descent, we measure the relative $L_1$ errors of the forward and backward passes, relative to the equivalent result after full convergence. We consider two starting points: 
the precomputed optimal solution $\mathbf{x}^a_0 = \mathbf{x}^{\star}$, and a uniform random vector $\mathbf{x}^b_0 = \bm{\eta} \sim \mathbf{U}(0,1)$. The former case is illustrated in Figure \ref{fig:unfold_pgd}, in which $\mathbf{x}_k$  remains stationary at each step.

Figure \ref{fig:fwd_bwd_err} reports the errors of the forward and backward pass at each iteration of the unfolded PGD under these two starting points. 
The figure shows that when starting the unfolding from the precomputed optimal solution $\mathbf{x}^a_0$, the forward pass error remains within error tolerance to zero. This is because $\mathbf{x}^{\star}(\mathbf{c}) \!=\! \mathcal{U}( \mathbf{x}^{\star}(\mathbf{c}), \mathbf{c} )$ is a \emph{fixed point} of (\ref{eq:opt_iteration}). Interestingly though, the backward pass also converges, but at a slightly faster rate than when starting from the random $\mathbf{x}^b_0$. 

We will see that this phenomenon holds in general: when an unfolded optimizer is iterated at a precomputed optimal solution, its backward pass converges. This has practical implications which can be exploited to improve the efficiency and modularity of differentiable optimization layers based on unrolling. These improvements will form the basis of our system for converting unrolled solvers to JgP-based implementations, called \emph{folded optimization}, and are discussed next.

\paragraph{Fixed-Point Unfolding: Forward Pass.}  
Note first that backpropagation by unfolding at a fixed point must assume that a fixed point has already been found. This is generally equivalent to finding a local optimum of the optimization problem which defines the forward-pass mapping \eqref{eq:opt_generic} \citep{beck2017first}. Since the calculation of the fixed point itself does not need to be backpropagated, it can be furnished by a \emph{blackbox} solver implementation. Furthermore, when $\mathbf{x}_0 = \mathbf{x}^{\star}$ is a fixed point of the iteration (\ref{eq:opt_iteration}), we have $\mathcal{U}(\mathbf{x}_{k}) = \mathbf{x}_{k} = \mathbf{x}^{\star}, \, \forall k$. Hence, \emph{there is no need to evaluate the forward pass} of $\mathcal{U}$ in each unfolded iteration of (\ref{eq:opt_iteration}) at $\mathbf{x}^{\star}$.

This enables the use  of any specialized method to compute the forward pass optimization \eqref{eq:opt_generic}, which can be different from unfolded algorithm used for backpropagation, assuming it shares the same fixed point. It also allows for highly optimized software implementations such as Gurobi \citep{gurobi}, and is a major advantage over existing differentiable optimization frameworks such as \texttt{cvxpy}, which requires converting the problem to a convex cone program before solving it with a specialized operator-splitting method for conic programming \citep{agrawal2019differentiable}, rendering it inefficient for many optimization problems.

\paragraph{Fixed-Point Unfolding: Backward Pass.}   

While the forward pass of each unfolded update step \eqref{eq:opt_iteration} need not be recomputed at a fixed point, the dotted curves of Figure \ref{fig:fwd_bwd_err} illustrate that its backward pass must still be iterated until convergence.  However, since $\mathbf{x}_{k} =  \mathbf{x}^{\star}$, we also have  $\frac{\partial \mathcal{U}(\mathbf{x}_{k})}{\partial \mathbf{x}_{k}} = \frac{\partial \mathcal{U}(\mathbf{x}^{\star})}{\partial \mathbf{x}^{\star}}$ at each iteration. Therefore the backward pass of $\mathcal{U}$ need only be computed \emph{once}, and iterated until backpropagation of the full optimization mapping \eqref{eq:opt_generic} converges. 

Next, it will be shown that this process is equivalent to iteratively solving a linear system of equations. We identify the iterative method first, and then the linear system it solves, before proceeding to prove this fact. The following textbook result can be found, e.g., in \citep{quarteroni2010numerical}.

\begin{lemma}
\label{lemma:linear-iteration}
Let $\mathbf{B} \in \mathbb{R}^{n \times n}$ and $\mathbf{b} \in \mathbb{R}^{n}$.  For any $\mathbf{z}_0 \in \mathbb{R}^n$, the iteration 
\begin{equation}
\label{eq:linear-iteration} \tag{LFPI}
        \mathbf{z}_{k+1}  =
         \mathbf{B} \mathbf{z}_{k} + 
        \mathbf{b} 
\end{equation}
converges to the solution $\mathbf{z}$ of the linear system
% \begin{equation}
% \label{eq:linear-system}
\( \mathbf{z} = \mathbf{B} \mathbf{z} +  \mathbf{b} \)
% \end{equation}
whenever $\mathbf{B}$ is nonsingular and has spectral radius
$\rho(\mathbf{B}) < 1$. 
Furthermore, the asymptotic convergence rate for $\mathbf{z}_k \to \mathbf{z}$ is 
\begin{equation}
\label{eq:convergence-rate}
    -\log\; \rho(\mathbf{B}) .
\end{equation}
\end{lemma}

\noindent 
Linear fixed-point iteration (LFPI) is a foundational iterative linear system solver, and can be applied to any linear system $\mathbf{A} \mathbf{x} \!=\! \mathbf{b}$ by rearranging 
$\mathbf{z} \!=\! \mathbf{B} \mathbf{z} \!+\! \mathbf{b}$ and identifying $\mathbf{A} \!=\! \mathbf{I} \!-\! \mathbf{B}$. 

Next, we exhibit the linear system which is solved for the desired gradients $\frac{\partial \mathbf{x}^{\star}}{\partial \mathbf{c}} (\mathbf{c})$ by unfolding at a fixed point. Consider the fixed-point  conditions of the iteration (\ref{eq:opt_iteration}):
\begin{equation}
    \label{eq:fixedpt_cond} \tag{FP}
    \mathbf{x}^{\star}(\mathbf{c}) = \mathcal{U}(\mathbf{x}^{\star}(\mathbf{c}), \; \mathbf{c} )
\end{equation}

Differentiating (\ref{eq:fixedpt_cond}) with respect to $\mathbf{c}$, 
\begin{equation}
\label{eq:diff_fixedpt_cond}
    \frac{\partial \mathbf{x}^{\star}}{\partial \mathbf{c}}(\mathbf{c}) = 
    \underbrace{{\frac{\partial \mathcal{U} }{\partial \mathbf{x}^{\star}} (\mathbf{x}^{\star}(\mathbf{c}), \; \mathbf{c} )}}_{\mathbf{\Phi}}
    	\cdot     
    \frac{\partial\mathbf{x}^{\star}}{\partial \mathbf{c}}(\mathbf{c}) + 
    \underbrace{\frac{\partial \mathcal{U} }{\partial \mathbf{c}} (\mathbf{x}^{\star}(\mathbf{c}), \; \mathbf{c} )}_{\mathbf{\Psi}}  
    ,
\end{equation}
by the chain rule and recognizing the implicit and explicit dependence of $\mathcal{U}$ on the independent parameters $\mathbf{c}$. Equation (\ref{eq:diff_fixedpt_cond}) will be called the \emph{differential fixed-point conditions}. 
Rearranging (\ref{eq:diff_fixedpt_cond}), the desired $\frac{\partial \mathbf{x}^{\star}}{\partial \mathbf{c}} (\mathbf{c})$ can be found in terms of $\mathbf{\Phi}$ and $\mathbf{\Psi}$ as defined above, to yield the system (\ref{eq:Lemma_fixedpt}) below.

The results discussed next are valid under the assumptions that $\mathbf{x}^{\star}\!\!:\!\mathbb{R}^n \!\to \!\mathbb{R}^n$ is differentiable in an open set $\mathcal{C}$, and Equation (\ref{eq:fixedpt_cond}) holds for $\mathbf{c} \in \mathcal{C}$. Additionally, $\mathcal{U}$ is assumed differentiable on an open set containing the point $(\mathbf{x}^{\star}( \mathbf{c} ), \mathbf{c})$. 

\begin{lemma}
\label{lemma:fixedpt-diff}
When $\mathbf{I}$ is the identity operator  and $\mathbf{\Phi}$ nonsingular, 
\begin{equation}
    \label{eq:Lemma_fixedpt}  \tag{DFP}
        (\mathbf{I} - \mathbf{\Phi} ) \frac{\partial \mathbf{x}^{\star}}{\partial \mathbf{c}} = \mathbf{\Psi} .
\end{equation}
\end{lemma}
The result follows from the Implicit Function theorem \citep{munkres2018analysis}. It implies that the Jacobian $\frac{\partial \mathbf{x}^{\star}}{\partial \mathbf{c}}$ can be found as the solution to a linear system once the prerequisite Jacobians $\mathbf{\Phi}$ and $\mathbf{\Psi}$ are found; these correspond to backpropagation of the update function $\mathcal{U}$ at $\mathbf{x}^{\star}(\mathbf{c})$.

\section{Folded Optimization}
\label{sec:FoldedOptimization}

We are now ready to discuss the central result of the paper. 
Informally, it states that the backward pass of an iterative solver (\ref{eq:opt_iteration}), unfolded at a precomputed optimal solution $\mathbf{x}^{\star}(\mathbf{c})$, is equivalent to solving the linear equations (\ref{eq:Lemma_fixedpt}) using linear fixed-point iteration, as outlined in Lemma \ref{lemma:linear-iteration}. 

This has significant implications for unrolling optimization. It shows that backpropagation of unfolding is computationally equivalent to solving linear equations using a specific algorithm and does not require automatic differentiation.
It also provides insight into the convergence properties of this backpropagation, including its convergence rate, and shows that more efficient algorithms can be used to solve (\ref{eq:Lemma_fixedpt}) in favor of its inherent LFPI implementation in unfolding.

The following results hold under the assumptions that 
the parameterized optimization mapping $\mathbf{x}^{\star}$ converges for certain parameters $\mathbf{c}$ through a sequence of iterates $\mathbf{x}_k(\mathbf{c}) \to \mathbf{x}^{\star}(\mathbf{c})$ using algorithm (\ref{eq:opt_iteration}), 
% and that $\mathbf{x}^{\star}$ is differentiable in an open set $\mathcal{C}$ with $\mathbf{c} \in \mathcal{C}$. 
% Additionally, $\mathcal{U}$ is differentiable in an open set containing the point $(\mathbf{x}^{\star}( \mathbf{c} ),; \mathbf{c})$ 
and that $\mathbf{\Phi}$ is nonsingular with a spectral radius $\rho(\mathbf{\Phi}) < 1$.

\begin{theorem}
\label{thm:unfolding_convergence_fixedpt}
% Assume that $\mathbf{x}^{\star}$  is differentiable on an open set $\mathcal{C}$, and that $\mathbf{c} \in \mathcal{C}$. Assume also that $\mathcal{U}$ is differentiable on an open set containing the point $(\mathbf{x}^{\star}( \mathbf{c} ),\; \mathbf{c})$.
% If $\mathbf{\Phi}$ is nonsingular and $\rho(\mathbf{\Phi}) < 1$, then an unfolding of algorithm (\ref{eq:opt_iteration}), from the starting point $\mathbf{x}_{k} =  \mathbf{x}^{\star}$, also converges in its backward pass to the unique solution of the linear system (\ref{eq:Lemma_fixedpt}), at an asymptotic rate of 
The backward pass of an unfolding of algorithm (\ref{eq:opt_iteration}), starting at the point $\mathbf{x}_{k} = \mathbf{x}^{\star}$, is equivalent to linear fixed-point iteration on the linear system (\ref{eq:Lemma_fixedpt}), and will converge to its unique solution   at an asymptotic rate of
\begin{equation}
    - \log \rho(\mathbf{\Phi}).
\end{equation}

\end{theorem}

\begin{proof}
Since $\mathcal{U}$ converges given any parameters $\mathbf{c} \in \mathcal{C}$,  Equation (\ref{eq:fixedpt_cond}) holds for any $\mathbf{c} \in \mathcal{C}$. Together with the assumption the $\mathcal{U}$ is differentiable on a neighborhood of $(\mathbf{x}^{\star}( \mathbf{c} ),\; \mathbf{c})$,
\begin{equation}
        (\mathbf{I} - \mathbf{\Phi} ) \frac{\partial \mathbf{x}^{\star}}{\partial \mathbf{c}} = \mathbf{\Psi}
\end{equation}
holds by Lemma \ref{lemma:fixedpt-diff}. 
When (\ref{eq:opt_iteration}) is unfolded, its backpropagation rule can be derived by differentiating its update rule:
\begin{subequations}
\begin{align}
        \label{eq:diff_update_1}
        \frac{\partial}{\partial \mathbf{c}} \left[ \; \mathbf{x}_{k+1}(\mathbf{c}) \;\right]  &=   \frac{\partial }{\partial \mathbf{c}}  \left[ \; \mathcal{U}(\mathbf{x}_k (\mathbf{c}) ,\;  \mathbf{c} )   \;\right] \\
        \label{eq:diff_update_2}
        \frac{\partial \mathbf{x}_{k+1}}{\partial \mathbf{c}}\; (\mathbf{c}) &=          \frac{\partial \mathcal{U}}{\partial \mathbf{x}_k} \frac{\partial \mathbf{x}_k}{\partial \mathbf{c}} + 
        \frac{\partial \mathcal{U}}{\partial \mathbf{c}},
\end{align}
\end{subequations}
where all terms on the right-hand side are evaluated at $\mathbf{c}$ and $\mathbf{x}_k (\mathbf{c})$. Note that in the base case $k=0$, since in general $\mathbf{x}_{0}$ is arbitrary and does not depend on $\mathbf{c}$, $\frac{\partial \mathbf{x}_0}{\partial \mathbf{c}} = \mathbf{0}$ and  
\begin{equation}
\label{eq:initial-update-diff}
        \frac{ \partial \mathbf{x}_{1}}{\partial \mathbf{c}}(\mathbf{c}) =
        \frac{\partial \mathcal{U}}{\partial \mathbf{c}}(\mathbf{x}_{0}, \mathbf{c}).
\end{equation}
This holds also when $\mathbf{x}_0 \!=\! \mathbf{x}^{\star}$ w.r.t.~backpropagation of (\ref{eq:opt_iteration}), since $\mathbf{x}^{\star}$ is precomputed outside the computational graph of its unfolding.  Now since $ \mathbf{x}^{\star}$ is a fixed point of (\ref{eq:opt_iteration}), 
\begin{equation}
\label{eq:xstar_all_k}
        \mathbf{x}_k(\mathbf{c}) = \mathbf{x}^{\star}(\mathbf{c})  \;\;\; \forall k \geq 0,
\end{equation}
which implies 
\begin{subequations}
\begin{align}
\label{eq:du_dxstar_all_k}
\frac{\partial \mathcal{U}}{\partial \mathbf{x}_k}(\mathbf{x}_k( \mathbf{c}), \; \mathbf{c} ) &= \frac{\partial \mathcal{U}}{\partial \mathbf{x}^{\star}}(\mathbf{x}^{\star}( \mathbf{c}), \; \mathbf{c} ) = \mathbf{\Phi}, \;\;\; \forall k \geq 0\\ 
\label{eq:du_dxstar_all_k2}
\frac{\partial \mathcal{U}}{\partial \mathbf{c}}(\mathbf{x}_k( \mathbf{c}), \; \mathbf{c} ) &= \frac{\partial \mathcal{U}}{\partial \mathbf{c}}(\mathbf{x}^{\star}( \mathbf{c}), \; \mathbf{c} ) = \mathbf{\Psi}, \;\;\; \forall k \geq 0.
\end{align}
\end{subequations}
Letting $\mathbf{J}_{k} \!\coloneqq\! \frac{\partial \mathbf{x}_k}{\partial \mathbf{c}}(\mathbf{c})$, the rule (\ref{eq:diff_update_2}) for unfolding at a fixed-point $\mathbf{x}^{\star}$ becomes, along with initial conditions (\ref{eq:initial-update-diff}), 
\begin{subequations}
\label{eq:proof_backprop}
\begin{align}
        \mathbf{J}_{0} &=  \mathbf{\Psi}\\
        \mathbf{J}_{k+1} & =
         \mathbf{\Phi} \mathbf{J}_{k} + 
        \mathbf{\Psi}.
\end{align}    
\end{subequations}
The result then holds by direct application of Lemma \ref{lemma:linear-iteration} to (\ref{eq:proof_backprop}), recognizing $\mathbf{z}_k = \mathbf{J}_{k}$ , $\mathbf{B} = \mathbf{\Phi}$  and  $\mathbf{z}_0 = \mathbf{b} = \mathbf{\Psi}$.
\end{proof}

\noindent The following is a direct result from the proof of Theorem~\ref{thm:unfolding_convergence_fixedpt}.
\begin{corollary}
Backpropagation of the fixed-point unfolding consists of the following rule:
\begin{subequations}
\label{eq:update-diff-fixed}
\begin{align}
        \mathbf{J}_{0} &=  \mathbf{\Psi}\\
        \mathbf{J}_{k+1} & =
         \mathbf{\Phi} \mathbf{J}_{k} + 
        \mathbf{\Psi},
\end{align}    
\end{subequations}
where  $\mathbf{J}_{k} \coloneqq \frac{\partial \mathbf{x}_k}{\partial \mathbf{c}}(\mathbf{c})$.
\end{corollary}

% Figure \ref{fig:comp_graph_3} of Appendix \ref{appendix:Figures} presents a simplified computational graph for unfolding three iterations of the optimization step (\ref{eq:opt_iteration}), starting at a precomputed $\mathbf{x}^{\star}$. The forward pass (in blue) repeatedly applies the update function $\mathcal{U}( \mathbf{x}(\mathbf{c}), \mathbf{c})$, while the backward pass (in red) uses a multivariate chain rule, as seen through the Jacobians $\frac{\partial \mathcal{U}}{\partial \mathbf{x}}$ and $\frac{\partial \mathcal{U}}{\partial \mathbf{c}}$. When (\ref{eq:opt_iteration}) is unfolded, the backward pass transformations $\frac{\partial \mathcal{U}}{\partial \mathbf{x}}$ and $\frac{\partial \mathcal{U}}{\partial \mathbf{c}}$ are modeled analytically and are not applied through unrolling.

Theorem \ref{thm:unfolding_convergence_fixedpt} specifically applies to the case where the initial iterate is the precomputed optimal solution, $\mathbf{x}_0 = \mathbf{x}^{\star}$. However, it also has implications for the general case where $\mathbf{x}_0$ is arbitrary. As the forward pass optimization converges, i.e. $\mathbf{x}_k \to \mathbf{x}^{\star}$ as $k \to \infty$, this case becomes identical to the one proved in Theorem \ref{thm:unfolding_convergence_fixedpt} and a similar asymptotic convergence result applies. If $\mathbf{x}_k \to \mathbf{x}^{\star}$ and $\mathbf{\Phi}$ is a nonsingular operator with $\rho(\mathbf{\Phi}) < 1$, the following result holds.
\begin{corollary}
\label{corr-unfolding-convergence}
When the parametric problem (\ref{eq:opt_generic}) can be solved by an iterative method of the form (\ref{eq:opt_iteration}) and the forward pass of the unfolded algorithm converges, the backward pass converges at an asymptotic rate that is bounded by $-\log \rho(\mathbf{\Phi})$.
\end{corollary}

\noindent The result above helps explain why the forward and backward pass in the experiment of Section \ref{sec:Unfolding_at_a_fixed_point} converge at different rates. If the forward pass converges faster than $-\log\; \rho(\mathbf{\Phi})$, the \emph{overall convergence rate of an unfolding is limited by that of the backward pass}.

\paragraph{Fixed-Point Folding.}
\label{sec:fixed_point_folding}
To improve efficiency, and building on the above findings, we propose to replace unfolding at the fixed point $\mathbf{x}^{\star}$ with the equivalent Jacobian-gradient product following the solution of (\ref{eq:Lemma_fixedpt}).
This leads to \emph{fixed-point folding}, a system for converting \emph{any} unrolled implementation of an optimization method (\ref{eq:opt_iteration}) into a \emph{folded optimization} that \emph{eliminates unrolling entirely}. By leveraging AD through a single step of the unrolled solver, but avoiding the use of AD to unroll through multiple iterations on the computational graph, it enables backpropagation of optimization layers by JgP using a seamless integration of automatic and analytical differentiation. Its modularization of the forward and backward passes, which are typically intertwined in unrolling, also allows for efficient blackbox implementations of each pass.

It is important to note that as per Definition \ref{def:unfolding}, the innermost optimization loop of a nested unrolling can be considered an unfolding and can be backpropagated by JgP with fixed-point folding. Subsequently, the next innermost loop can now be considered unfolded and the same process applied until all unrolled optimization loops are replaced with their analytical models. Figure \ref{fig:unfolding_scheme} depicts fixed-point folding,  where the gray arrows symbolize a blackbox forward pass and the long red arrows illustrate that a backpropagation is performed an iterative linear system solver. %\rev{Point to the appendix for a discussion about stepsize}
The procedure is also exemplified by \textit{f-PGDb} (introduced in Section 6), which applies successive fixed-point folding through ADMM and PGD to compose a JgP-based differentiable layer for any optimization problem with a smooth objective function and linear constraints.

Note that although it is not used for forward pass convergence, a folded optimizer still typically requires selecting a constant stepsize, or similar parameter, to specify $\mathcal{U}$ and the resulting Jacobian model \eqref{eq:Lemma_fixedpt}. This can affect $\rho(\mathbf{\Phi})$, and hence the backward pass convergence and its rate by Theorem~\ref{thm:unfolding_convergence_fixedpt}. A further discussion of the aspect is made in Appendix \ref{appendix:stepsize}.

\begin{figure*}
    \centering
    \vspace{-2pt}
     \includegraphics[width=0.32\textwidth]{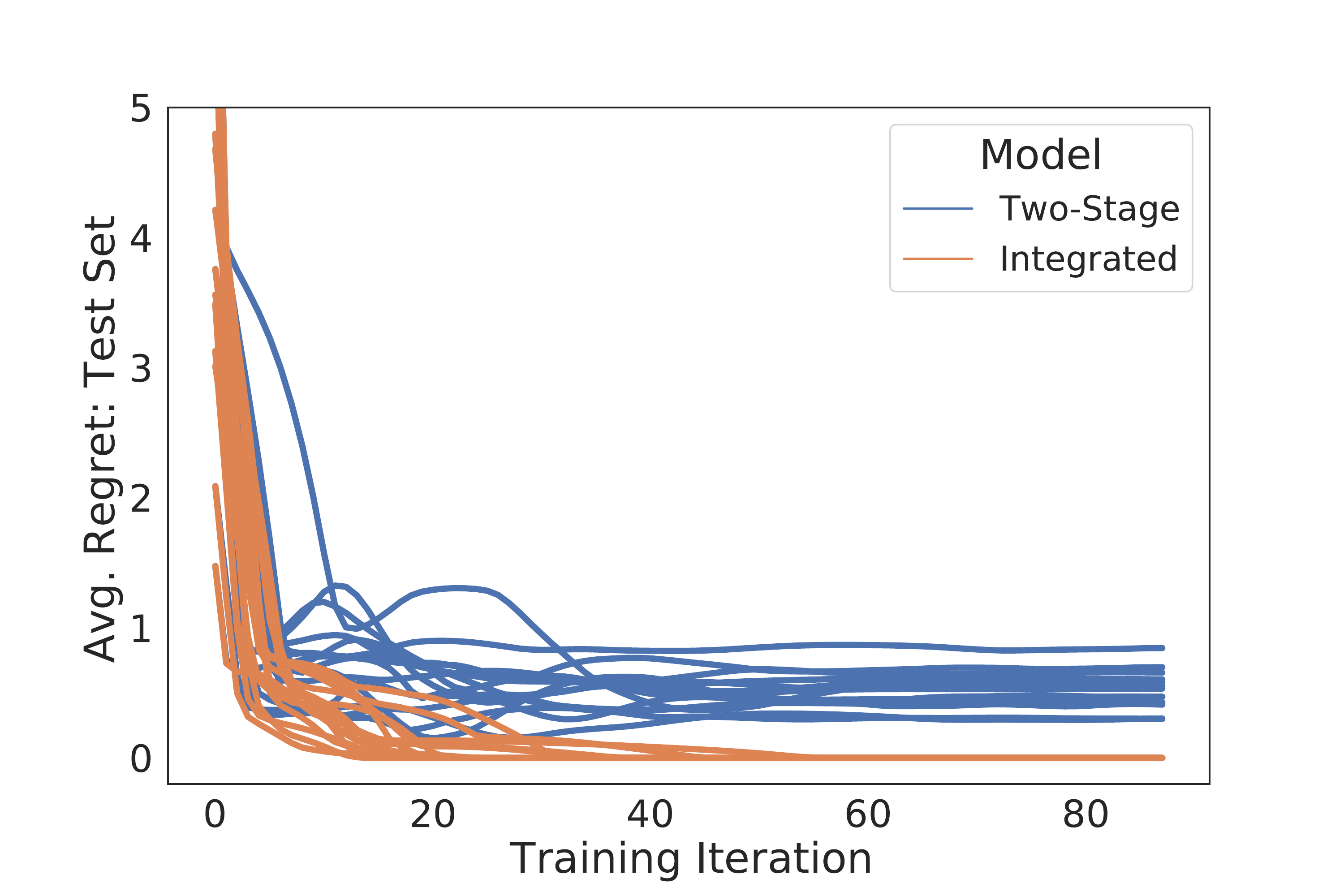}
   \includegraphics[width=0.32\linewidth]{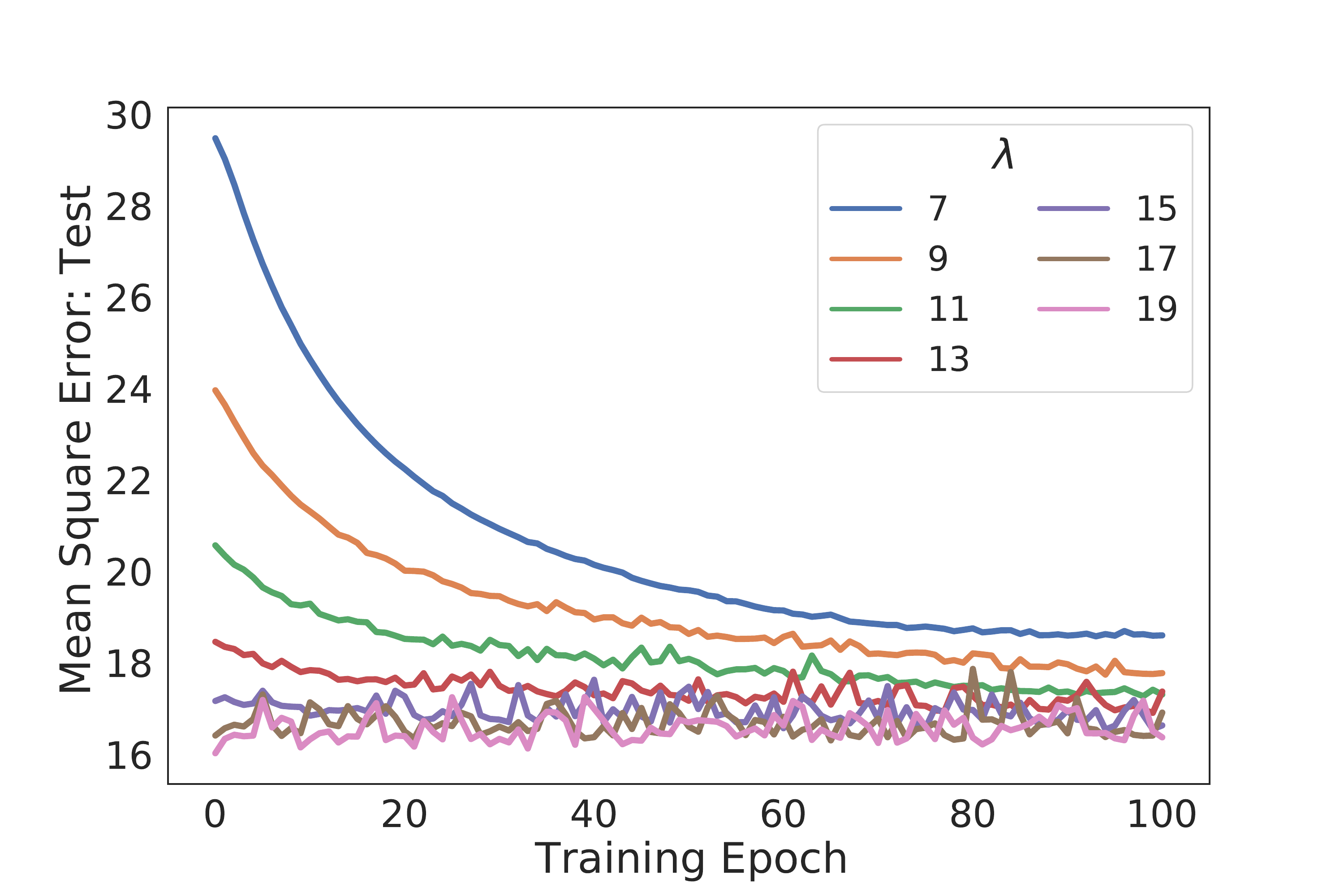} 
    \includegraphics[width=0.32\linewidth]{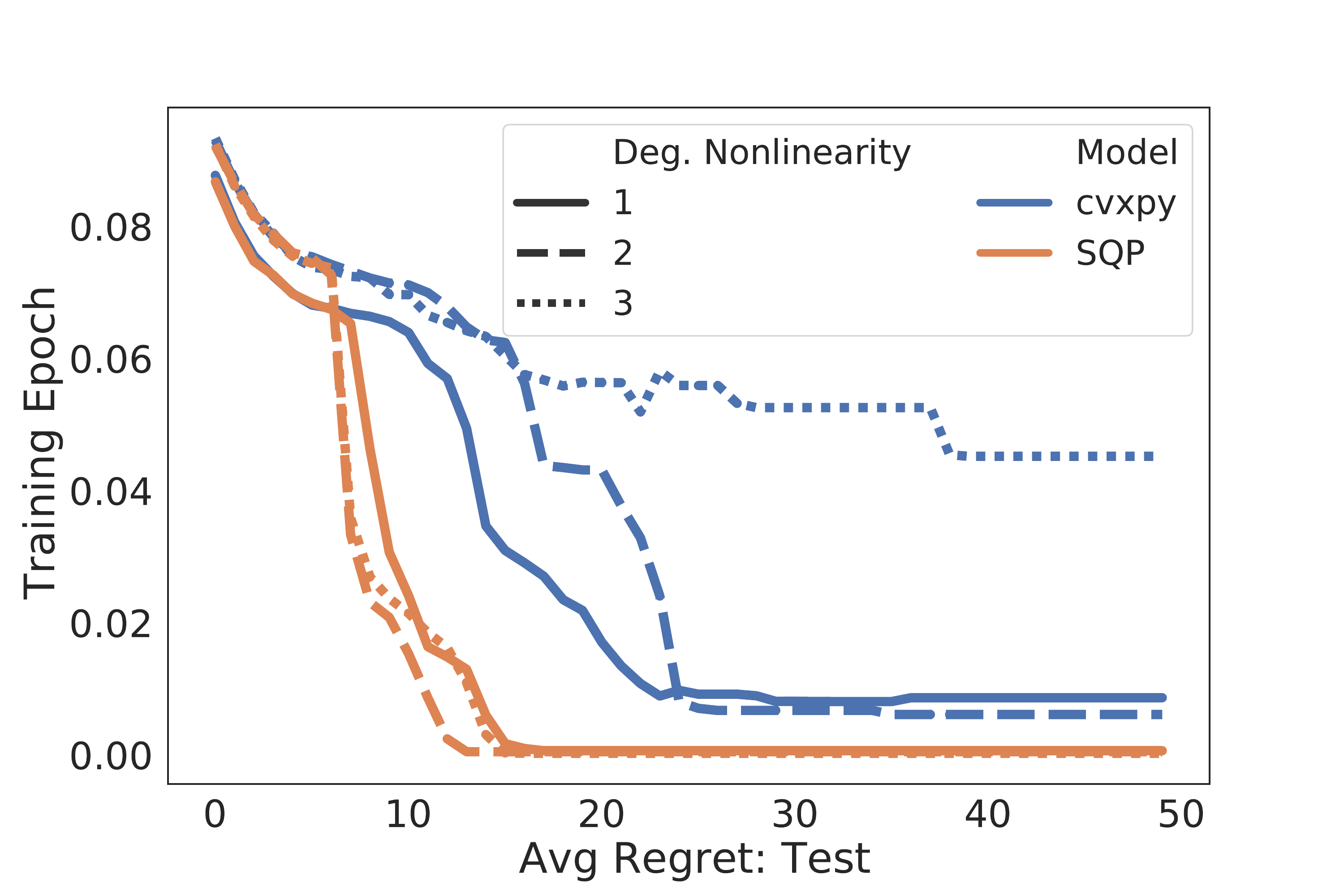}
   \\%[-4pt]
  {\small \hspace{0pt} (a) \hspace{.3\linewidth} (b) \hspace{.3\linewidth} 
  (c) \vspace{-0pt}}
    \caption{Bilinear decision focus (a),  Enhanced Denoising with \textit{f-FDPG} (b), and Portfolio optimization (c).}
    \label{fig:results}
\end{figure*}

\section{Experiments}
\label{sec:Experiments}

This section evaluates folded optimization %and four \texttt{fold-opt} implementations 
on four end-to-end optimization and learning tasks. It is primarily evaluated against \texttt{cvxpy}, which is the preeminent general-purpose differentiable optimization solver. Two crucial limitations of \texttt{cvxpy} are its efficiency and expressiveness. This is due to its reliance on transforming general optimization programs to convex cone programs, before applying a standardized operator-splitting cone program solver and differentiation scheme (see Appendix \ref{app:RelatedWork}). 
This precludes the incorporation of problem-specific solvers in the forward pass  and limits its use to convex problems only. 
One major benefit of \texttt{fold-opt} is the modularity of its forward optimization pass, which can apply any black-box algorithm to produce $\mathbf{x}^{\star}(\mathbf{c})$. In each experiment below, this is used to demonstrate a different advantage. 

A summary of results is provided below for each study, and a more complete specification is provided in Appendix \ref{appendix:experimental}.

\paragraph{Implementation details.}
All the folded optimizers used in this section were produced using the accompanying Python library \texttt{fold-opt}, which supplies routines for constructing and solving the system (\ref{eq:Lemma_fixedpt}), and integrating the resulting Jacobian-vector products into the computational graph of PyTorch. To do so, it requires a Pytorch implementation of an update function $\mathcal{U}$ for an appropriately chosen optimization routine.  The linear system  (\ref{eq:Lemma_fixedpt}) is be solved by a user-specifed blackbox linear solver, as is the forward-pass optimization solver, as discussed in Section \ref{sec:Unfolding_at_a_fixed_point}. Implementation details of \texttt{fold-opt} can be found in Appendix \ref{appendix:implementation}.

The experiments test four folded optimizers: 
\textbf{(1)} \textit{f-PGDa} applies to optimization mappings with linear constraints, and is based on folding projected gradient descent steps, %(Equation \ref{eq:update-pgd}), 
where each inner projection is a QP solved by the differentiable QP solver \texttt{qpth} \citep{amos2019optnet}. 
\textbf{(2)} \textit{f-PGDb} is a variation on the former, in which the inner QP step is differentiated by fixed-point folding of the ADMM solver detailed in Appendix \ref{appendix:models}. 
\textbf{(3)} \textit{f-SQP} applies to optimization with nonlinear constraints and uses folded SQP with the inner QP differentiated by \texttt{qpth}. 
\textbf{(4)} \textit{f-FDPG} comes from fixed-point folding of the Fast Dual Proximal Gradient Descent (FDPG) shown in Appendix \ref{appendix:models}. 
The inner $\operatorname{Prox}$ is a soft thresholding operator, whose simple closed form is differentiated by AD in PyTorch.

\paragraph{Decision-focused learning with nonconvex bilinear programming.}
%Nonconvex bilinear programming.}
\label{subsec:bilinear}
% It is common in decision-focused settings to learn operational decisions by predicting the objective coefficients of a linear
% problem. 
The first experiment showcases the ability of folded optimization to be applied in decision-focused learning with \emph{nonconvex} optimization. 
In this experiment, we predict the coefficients of a \emph{bilinear} program
% \begin{subequations}
% \label{eq:bilinear}
\begin{align*}
    \mathbf{x}^{\star}(\mathbf{c}, \mathbf{d}) = \argmax_{\mathbf{0} \leq \mathbf{x}, \mathbf{y}  \leq \mathbf{1} } &\;\;
    \mathbf{c}^T \mathbf{x} + \mathbf{x}^T \mathbf{Q} \mathbf{y} +\mathbf{d}^T \mathbf{y}    \\
    \texttt{s.t.} \;\; 
    & \sum \mathbf{x} = p, \; \sum \mathbf{y} = q,
\end{align*}
%\end{subequations}
in which two separable linear programs are confounded by a nonconvex quadratic objective term $\bm{Q}$. Costs $\bm{c}$ and $\bm{d}$ are predicted by a 5-layer network, while $p$ and $q$ are constants.
Such programs have numerous industrial applications such as optimal mixing and pooling in gas refining \citep{audet2004pooling}. 
Here we focus on the difficulty posed by the problem's form and propose a task to evaluate \textit{f-PGDb} in learning with nonconvex optimization. Feature and cost data are generated by the process described in Appendix \ref{appendix:experimental}, 
along with $15$ distinct $\mathbf{Q}$ for a collection of nonconvex decision models.

It is known that PGD converges to local optima in nonconvex problems \citep{attouch2013convergence}, and this folded implementation uses the Gurobi nonconvex QP solver to find a global optimum. Since no known general framework can accommodate nonconvex optimization mappings in end-to-end models,  we benchmark against the \emph{two-stage} approach, in which the costs $\mathbf{c}$, and $\mathbf{d}$ are targeted to ground-truth costs by MSE loss and the optimization problem is solved as a separate component from the learning task (see Appendix \ref{appendix:DFL} 
for additional details). The integrated  \textit{f-PGDb} model minimizes solution regret (i.e., suboptimality) directly.
\citep{elmachtoub2020smart}. Notice  in Figure \ref{fig:results}(a) how \textit{f-PGDb} achieves much lower regret for each of the $15$ nonconvex objectives.

\paragraph{Enhanced Total Variation Denoising.}
\label{subsec:denoising_experiment}
This experiment illustrates the efficiency benefit of incorporating problem-specific solvers. The optimization models a denoiser 
\[
    \mathbf{x}^{\star}(\mathbf{D}) = \argmin_{\mathbf{x}} \;\;
    \frac{1}{2} \| \mathbf{x}-\mathbf{d} \|^2 +  \lambda \| \mathbf{D} \mathbf{x} \|_1,
\]
which seeks to recover the true signal $\mathbf{x}^{\star}$ from  a noisy input $\mathbf{d}$ and is often best handled by variants of Dual Proximal Gradient Descent. Classically, $\mathbf{D}$ is a differencing matrix so that $\| \mathbf{D} \mathbf{x} \|_1$ represents total variation. %Recent works \citep{monga2021algorithm} aim to learn $\mathbf{D}$ by unrolling, for more accurate denoising on some input distribution.
Here we initialize $\mathbf{D}$ to this classic case and \emph{learn} a better $\mathbf{D}$ by targeting a set of true signals with MSE loss and adding Gaussian noise to generate their corresponding noisy inputs. Figure \ref{fig:results}(b) shows test MSE throughout training due to \textit{f-FDPG} for various choice of $\lambda$.  Appendix \ref{appendix:Figures} 
shows comparable results from the framework of \cite{amos2019optnet}, which converts the problem to a QP form (see Appendix \ref{appendix:models}) 
in order to differentiate the mapping analytically with \texttt{qpth}. Small differences in these results likely stem from solver error tolerance in the two methods.
%We also test \textit{f-FDPG} which employs FDPG in its forward pass. 
%Following \citep{amos2019optnet}, one-dimensional signals act as training targets while random noise is applied to create their corresponding noisy input data. The parameters $\mathbf{D}$ are trained directly to minimize MSE loss and in each case, $\mathbf{D}$ is initialized to the classic differencing operator. Test set loss throughout training is shown in Figure \ref{fig:denoiser_plots}. The baseline model used here is \texttt{qpth}, and both models produce comparable results. 
However, \emph{\textit{f-FDPG} computes $\mathbf{x}^{\star}(\mathbf{D})$ up to $40$ times faster}. 

\paragraph{Mutilabel Classification on CIFAR100.}
\label{subsec:multilabel_experiment}
Since gradient errors accumulate at each training step, we ask how precise are the operations performed by \texttt{fold-opt} in the backward pass. This experiment compares the backpropagation of both \textit{f-PGDa} and \textit{f-SQP}  with that of \texttt{cvxpy}, by using the forward pass of \texttt{cvxpy} in each model as a control factor. 

% Next, we focus on measuring the accuracy of the backpropagation step produced by the proposed folded optimization framework. 
This experiment, adapted from \citep{Berrada2018SmoothLF}, implements a smooth top-5 classification model on noisy CIFAR-100. The optimization below maps image feature embeddings $\mathbf{c}$ from DenseNet  $40$-$40$ \citep{huang2017densely}, to smoothed top-$k$ binary class indicators (see Appendix \ref{appendix:experimental} 
for more details): 
\begin{align}
\label{eq:topk-lp}
    \mathbf{x}^{\star}(\mathbf{c}) \!=\! 
    \argmax_{\mathbf{0} \leq \mathbf{x} \leq \mathbf{1}} &\;\;
    \mathbf{c}^T \mathbf{x} + \sum_i x_i \log x_i \;\;
    \textsl{s.t.} \;
    \sum \mathbf{x} = k 
\end{align}
Appendix \ref{appendix:Figures} 
shows that all three models have indistinguishable classification accuracy throughout training, indicating the backward pass of both \texttt{fold-opt} models is precise and agrees with a known benchmark even after 30 epochs of training on 45k samples. On the other hand, the more sensitive test set shows marginal accuracy divergence after a few epochs.

\paragraph{Portfolio Prediction and Optimization.}
\label{subsec:portfolio_experiment}
Having established the equivalence in performance of the backward pass across these models, the final experiment describes a situation in which \texttt{cvxpy} makes non negligible errors in the forward pass of a problem with nonlinear constraints:
\begin{align}
\label{eq:porfolio}
    \mathbf{x}^\star(\mathbf{c}) = \argmax_{\mathbf{0} \leq \mathbf{x} } 
    \mathbf{c}^T \mathbf{x}\;\;
    \textsl{s.t.} \;
     \mathbf{x}^T \mathbf{V} \mathbf{x} \leq \gamma, \; \sum \mathbf{x} = 1.
\end{align}
This model describes a risk-constrained portfolio optimization where $\bm{V}$ is a covariance matrix, and the predicted cost coefficients $\mathbf{c}$ represent assets prices \citep{elmachtoub2020smart}. % Additional details about the model and dataset, are once again reported in Appendix \ref{appendix:experimental}. 
A  $5$-layer ReLU network is used to predict future prices $\bm{c}$ from exogenous feature data, and trained to minimize regret (the difference in profit between optimal portfolios under predicted and ground-truth prices) by integrating Problem (\ref{eq:porfolio}).
% In this task, a collection of assets is considered and their future price predicted from empirical input features by a five-layer ReLU network. Synthetic data is generated by the method proposed in \cite{elmachtoub2020smart}, which uses mappings of various degrees of nonlinearity to generate prices from input features, as detailed in Appendix \ref{appendix:experimental}. The optimization model (\ref{eq:portfolio}) maximizes the expected profit based on predicted prices $\mathbf{c}$, under a (quadratic) constraint on the overall covariance of the portfolio, which is a measure of risk. 
% The goal is to maximize expected profit, relative to ground-truth prices. Equivalently, we minimize regret, the difference in profit between optimal portfolios under predicted and ground-truth prices (see Appendix). 
The folded \textit{f-SQP} layer used for this problem employs Gurobi QCQP solver in its forward pass. This again highlights the ability of \texttt{fold-opt} to accommodate a highly optimized blackbox solver. 
Figure \ref{fig:results}(c)   shows test set regret  throughout training, three synthetically generated datasets of different nonlinearity degrees. 
Notice the accuracy improvements of  \texttt{fold-opt} over ~\texttt{cvxpy}. 
Such dramatic differences can be explained by non-negligible errors made in \texttt{cvxpy}'s forward pass optimization on some problem instances, which occurs regardless of error tolerance settings (please see Appendix \ref{appendix:experimental} 
for details). In contrast, Gurobi agrees to machine precision with a custom SQP solver, and solves about $50 \%$ faster than \texttt{cvxpy}. This shows the importance of highly accurate optimization solvers for accurate end-to-end training.

\section{Conclusions}
This paper introduced folded optimization, a framework for generating analytically differentiable optimization solvers from unrolled implementations. Theoretically, folded optimization was justified by a novel analysis of unrolling at a precomputed optimal solution, which showed that its backward pass is equivalent to solution of a solver's differential fixed-point conditions, specifically by fixed-point iteration on the resulting linear system. This allowed for the convergence analysis of the backward pass of unrolling, and evidence that the backpropagation of unrolling can be improved by using superior linear system solvers. The paper showed that folded optimization offers substantial advantages over existing differentiable optimization frameworks, including modularization of the forward and backward passes and the ability to handle nonconvex optimization.

\section*{Acknowledgements}
This research is partially supported by NSF grant 2232054 and NSF CAREER Award 2143706. Fioretto is also supported by an Amazon Research Award and a Google Research Scholar Award. Its views and conclusions are those of the authors only.

\appendix
\section{Related Work}
\label{app:RelatedWork}
This section categorizes end-to-end optimization and learning approaches into those based on \emph{unrolling}, and \emph{analytical} differentiation. Since this paper focuses on converting unrolled implementations into analytical ones, each category is reviewed first below.

\paragraph{Unrolling optimization algorithms.}
Automatic Differentiation (AD) is  the primary method of backpropagating gradients in deep learning models for training with stochastic gradient descent.  Modern machine learning frameworks such as PyTorch have natively implemented differentiation rules for a variety of functions that are commonly used in deep models, as well as interfaces to define custom differentiation rules for new functions \citep{paszke2017automatic}. As a mainstay of deep learning, AD is also a natural tool for backpropagating through constrained optimization mappings. \emph{Unrolling} refers to the execution of an optimization algorithm, entirely on the computational graph, for backpropagation by AD from the resulting optimal solution to its input parameters. Such approaches are general and apply to a broad range of optimization models. They  can be performed simply by implementing a solution algorithm within an AD framework, without the need for analytical modeling of an optimization mapping's derivatives \citep{domke2012generic}. However, unrolling over many iterations has been shown to encounter issues of time and memory inefficiency due to the size of its computational graph \citep{amos2019optnet}. Further issues encountered in unrolling, such as vanishing and exploding gradients, are reminiscent of recurrent neural networks \citep{monga2021algorithm}. On the other hand, unrolling may offer some unique practical advantages, like the ability to learn optimization parameters such as stepsizes to accelerate the solution of each optimization during training \citep{shlezinger2022model}.

\paragraph{Analytical differentiation of optimization models.}
% Distinguish between differentiation and smoothing; use the LP lesson to set up for experiments if possible

Differentiation through constrained argmin problems  in the context of machine learning was discussed as early as \cite{gould2016differentiating}, who proposed first to implicitly differentiate the argmin of a smooth, unconstrained convex function by its first-order optimality conditions, defined when the gradient of the objective function equals zero. This technique is then extended to find approximate derivatives for constrained problems, by applying it to their unconstrained log-barrier approximations. Subsequent approaches applied implicit differentiation to the KKT optimality conditions of constrained problems directly \citep{amos2019optnet,amos2019limited}, but only on special problem classes such as  Quadratic Programs. \cite{konishi2021end} extend the method of \cite{amos2019optnet}, by modeling second-order derivatives of the optimization for training with gradient boosting methods. \cite{donti2017task} uses the differentiable quadratic programming solver of \citep{amos2019optnet} to approximately differentiate general convex programs through quadratic surrogate problems.   Other problem-specific approaches to analytical differentiation models include ones for sorting and ranking \citep{blondel2020fast}, linear programming \citep{mandi2020interior}, and convex cone programming \citep{agrawal2019differentiating}.

The first general-purpose differentiable optimization solver was proposed in \cite{agrawal2019differentiable}, which leverages the fact that any convex program can be converted to a convex cone program \citep{nemirovski2007advances}. The equivalent cone program is subsequently solved and differentiated following \cite{agrawal2019differentiating}, which implicitly differentiates a zero-residual condition representing optimality \citep{busseti2019solution}. A differentiable solver library \texttt{cvxpy} is based on this approach, which converts convex programs to convex cone programs by way of their graph implementations as described in \cite{grant2008graph}. The main advantage of the system is that it applies to any convex program and has a simple symbolic interface. A major disadvantage is its restriction to solving problems only in a standard convex cone form with an ADMM-based conic programming solver, which performs poorly on some problem classes, as seen in Section \ref{sec:Experiments}. 

A related line of work concerns end-to-end learning with \emph{discrete} optimization problems, which includes linear programs, mixed-integer programs and constraint programs. These problem classes often define discontinuous mappings with respect to their input parameters, making their true gradients unhelpful as descent directions in  optimization. Accurate end-to-end training can be achieved by \emph{smoothing} the optimization mappings, to produce approximations which yield more useful gradients. A common approach is to augment the objective function with smooth regularizing terms such as euclidean norm or entropy functions \citep{wilder2019melding,ferber2020mipaal,mandi2020interior}. Others show that similar effects can be produced by applying random noise to the objective \citep{berthet2020learning,paulus2020gradient}, or through finite difference approximations \citep{poganvcic2019differentiation,sekhar2022gradient}. This enables end-to-end learning with discrete structures such as constrained ranking policies \citep{kotary2022end}, shortest paths in graphs \citep{elmachtoub2020smart}, and various decision models \citep{wilder2019melding}.

\section{Implementation Details}
\label{appendix:implementation}

The purpose of the \texttt{fold-opt} library is to facilitate the conversion of unfolded optimization code into JgP-based differentiable optimization, by leveraging automatic differentiation in Pytorch. It relies on the fact that backpropagation of a (gradient) vector $\bm{g}$ through the computational graph of a function $\bm{x} \rightarrow \bm{f}(\bm{x})$ by reverse-mode automatic differentiation is equivalent to computing the JgP product $\bm{g} \cdot \frac{\partial \bm{f}(\bm{x}) }{\partial \bm{x} }$.  

In principle, the following steps are required:  \textbf{(1)} After executing a blackbox optimization, initialize the  optimal solution $\mathbf{x}^{\star}$ onto the computational graph of PyTorch. \textbf{(2)} Execute a single step of the unrolled loop's update function to get $\mathbf{x}^{\star \star}(\mathbf{c}) =  \mathcal{U}(\mathbf{x}^{\star}, \; \mathbf{c})$ and save its computational graph; in principle, the forward execution can be avoided given its known result $\mathbf{x}^{\star}$. \textbf{(3)} Backpropagate each column of the identity matrix from $\mathbf{x}^{\star \star}(\mathbf{c})$ to $\mathbf{x}^{\star}$ and from $\mathbf{x}^{\star \star}(\mathbf{c})$ to $\mathbf{c}$ to assemble $\mathbf{\Phi} \coloneqq {\frac{\partial \mathcal{U} }{\partial \mathbf{x}^{\star}} (\mathbf{x}^{\star}(\mathbf{c}), \; \mathbf{c} )}$ and $\mathbf{\Psi} \coloneqq {\frac{\partial \mathcal{U} }{\partial \mathbf{c}} (\mathbf{x}^{\star}(\mathbf{c}), \; \mathbf{c} )}$, respectively (see Section \ref{sec:FoldedOptimization}).  \textbf{(4)} Solve equation  $(\mathbf{I} - \mathbf{\Phi} ) \frac{\partial \mathbf{x}^{\star}}{\partial \mathbf{c}} = \mathbf{\Psi}$  for the Jacobian $\frac{\partial \bm{x}^{\star} }{\partial \mathbf{c}}(\bm{c})$ using a linear system solver, and apply the Jacobian-vector product $\frac{\partial \mathcal{L}}{\partial \mathbf{c}} = \frac{\partial \mathcal{L}}{\partial \mathbf{x}^{\star}} \cdot \frac{\partial \mathbf{x}^{\star}(\mathbf{c})}{\partial \mathbf{c}}$ to backpropagate incoming gradients. 

In practice, since only the Jacobian-gradient product $\frac{\partial \mathcal{L}}{\partial \mathbf{c}} = \frac{\partial \mathcal{L}}{\partial \mathbf{x}^{\star}} \cdot \frac{\partial \mathbf{x}^{\star}(\mathbf{c})}{\partial \mathbf{c}}$ is required for backpropagation,  the above steps \textbf{(3)} and \textbf{(4)}  are computationally superflous. It is more 
efficient to solve a related linear system directly for the vector $\frac{\partial \mathcal{L}}{\partial \mathbf{c}} = \frac{\partial \mathcal{L}}{\partial \mathbf{x}^{\star}} \cdot \frac{\partial \mathbf{x}^{\star}(\mathbf{c})}{\partial \mathbf{c}}$ . Furthermore, the linear system can be solved by iterative methods without explicitly constructing the matrices $\mathbf{\Phi}$ and $\mathbf{\Psi}$, by simulating their left-sided JgP's using reverse-mode AD though $\mathcal{U}$. To see how, write the backpropagation of the loss gradient $\frac{\partial \mathcal{L}}{\partial \mathbf{x}^{\star}}$ through $k$ unfolded steps of \eqref{eq:opt_iteration} at the fixed point $\bm{x}^{\star}$ as 
\begin{equation}
\label{eq:backprop_goal}
\frac{\partial \mathcal{L}}{\partial \mathbf{x}^{\star}}^T \left(   \frac{\partial \mathbf{x}^{k}(\mathbf{c})}{\partial \mathbf{c}}  \right) .
\end{equation}
We seek to compute the limit $\frac{\partial \mathcal{L}}{\partial \mathbf{c}} = \bm{g}^T \bm{J}$ where $\bm{g} = \frac{\partial \mathcal{L}}{\partial \mathbf{x}^{\star}}$, $\bm{J} \coloneqq \lim_{k \rightarrow \infty} \bm{J}_k  $  , and  $\bm{J}_k =  \frac{\partial \mathbf{x}^{k}(\mathbf{c})}{\partial \mathbf{c}} $  . Following the backpropagation rule \eqref{eq:update-diff-fixed}, the expression \eqref{eq:backprop_goal} is equal to 
\begin{subequations}
\begin{align}
    \bm{g}^T \bm{J}_k &=  \bm{g}^T \left( \mathbf{\Phi} \bm{J}_{k-1} + \mathbf{\Psi} \right) \\
    &=  \bm{g}^T \left( \mathbf{\Phi}^k \mathbf{\Psi} + \mathbf{\Phi}^{k-1} \mathbf{\Psi}  + \ldots + \mathbf{\Phi} \mathbf{\Psi}  + \mathbf{\Psi}   \right)
\end{align}
\end{subequations}

This expression can be rearranged as
\begin{equation}
\label{eq:rearrange}
    \bm{g}^T \bm{J}_k  =   \bm{v}_k^T \mathbf{\Psi} 
\end{equation}

\noindent where
\begin{equation}
\label{eq:def_v}
 \bm{v}_k^T \coloneqq \left(\bm{g}^T  \mathbf{\Phi}^k  + \bm{g}^T \mathbf{\Phi}^{k-1}   + \ldots + \bm{g}^T  \mathbf{\Phi}   + \bm{g}^T    \right)   .
\end{equation}
The sequence $\bm{v}_k$ can be computed most efficiently as
\begin{equation}
\label{eq:recursion_v}
  \bm{v}_k^T = \bm{v}_{k-1}^T \mathbf{\Phi} + \bm{g}^T  
\end{equation}
which identifies $ \bm{v} \coloneqq  \lim_{k \rightarrow \infty} \bm{v}_k  $ as the solution of the linear system
\begin{equation}
\label{eq:v_system}
  \bm{v}^T (\bm{I} - \mathbf{\Phi})   =   \bm{g}^T  
\end{equation}
under the conditions of Lemma \eqref{lemma:linear-iteration}, after transposing both sides of \eqref{eq:recursion_v} and \eqref{eq:v_system} . 

Once $\bm{v}^T$ is calculated by \eqref{eq:recursion_v}, the desired JgP is
\begin{equation}
\label{eq:gTJ}
  \bm{g}^T \bm{J} = \bm{v}^T \mathbf{\Psi}  .
\end{equation}
The left matrix-vector product with respect to  $\mathbf{\Phi}$ in \eqref{eq:recursion_v} and $\mathbf{\Psi}$ in \eqref{eq:gTJ} can be computed by backpropagation through the computational graph of the update function $\mathcal{U}(\bm{x}^{\star}(\bm{c}), \; \bm{c})$, backward  to $\bm{x}^{\star}(\bm{c})$ and $\bm{c}$ respectively.

Notice that in contrast to unfolding, this backpropagation method requires to store the computational graph only for a single update step, rather than for an entire  optimization routine consisting of many iterations.

Having reduced the calculation of $\bm{g}^T \bm{J}$ to the solution of a linear system \eqref{eq:v_system} followed by a matrix-vector product \eqref{eq:gTJ}, it is clear how efficiency can be improved by replacing the LFPI iterations \eqref{eq:recursion_v} with a faster-converging linear solution scheme based on matrix-vector products, such as Krylov subspace methods. This emphasizes the inherently sub-optimal convergence rate of backpropagation in unfolded solvers, and such upgrades will be planned for future versions of \texttt{fold-opt}.

\section{Optimization Models}
\label{appendix:models}

\paragraph{Soft Thresholding Operator}

The soft thresholding operator defined below arises in the solution of denoising problems proximal gradient descent variants as the proximal operator to the $\| \cdot \|_1$ norm:
\[
    \mathcal{T}_{\lambda}(\mathbf{x}) = \left[  | \mathbf{x} | - \lambda \mathbf{e}  \right]_{+}  \cdot \textit{sgn} (\mathbf{x})
\]

\paragraph{Fast Dual Proximal Gradient Descent}
The following is an FDPG implementation from \cite{beck2017first}, specialized to solve the denoising problem 
\[
    \mathbf{x}^{\star}(\mathbf{D}) = \argmin_{\mathbf{x}} \;\;
    \frac{1}{2} \| \mathbf{x}-\mathbf{d} \|^2 +  \lambda \| \mathbf{D} \mathbf{x} \|_1,
\]
of Section \ref{sec:Experiments}. Letting $\mathbf{u}_k$ be the primal solution iterates, with $t_0=1$ and arbitrary $\mathbf{w}_0 = \mathbf{y}_0$:
\begin{subequations}
\begin{align}
    \mathbf{u}_k &= \mathbf{D}^T \mathbf{w}_k + \mathbf{d} \\
    \mathbf{y}_{k+1} &= \mathbf{w}_k - \frac{1}{4} \mathbf{D} \mathbf{u}_k + \frac{1}{4} \mathcal{T}_{4 \lambda} ( \mathbf{D} \mathbf{u}_k - 4 \mathbf{w}_k ) \\
    t_{k+1} &= \frac{1 + \sqrt{1+4 t_k^2}}{2} \\
    \mathbf{w}_{k+1} &= \mathbf{y}_{k+1} + \left( \frac{t_k - 1}{t_{k+1}} \right) (\mathbf{y}_{k+1} - \mathbf{y}_k)
\end{align}
\end{subequations}

\paragraph{Quadratic Programming by ADMM.}
A Quadratic Program is an optimization problem with convex quadratic objective and linear constraints. The following ADMM scheme of \cite{boyd2011distributed} solves any quadratic programming problem of the standard form:
\begin{subequations}
\begin{align}
    \argmax_{x} &\;\;
    \frac{1}{2} \mathbf{x}^T \mathbf{Q} \mathbf{x} + \mathbf{p}^T \mathbf{x}\\
    \textit{s.t.} \;\; & \mathbf{A}\mathbf{x} = \mathbf{b}\\
     &\mathbf{x} \geq \mathbf{0}
\end{align}
\end{subequations}
by declaring the operator splitting
\begin{subequations}
\begin{align}
    \argmax_{\mathbf{x}} &\;\;
    f(\mathbf{x}) + g(\mathbf{z})\\
    \textit{s.t.} \;\; & \mathbf{x}=\mathbf{z}
\end{align}
\end{subequations}
with $  f(\mathbf{x}) = \frac{1}{2} \mathbf{x}^T \mathbf{Q} \mathbf{x} + \mathbf{p}^T \mathbf{x}  $, $dom(f) = \{ \mathbf{x}: \mathbf{A}\mathbf{x} = \mathbf{b} \} $, $  g(\mathbf{x}) = \delta( \mathbf{x} \geq 0 )  $ and where $\delta$ is the indicator function. 

This results in the following ADMM iterates: 
\begin{enumerate}
    \item Solve $\begin{bmatrix} \mathbf{P}+\rho \mathbf{I}&\mathbf{A}^T\\\mathbf{A}&\mathbf{0}\end{bmatrix} \begin{bmatrix} \mathbf{x}_{k+1} \\ \pmb{\nu}\end{bmatrix}   =  \begin{bmatrix}-\mathbf{q} + \rho(\mathbf{z}_k - \mathbf{u}_k)\\\mathbf{b}\end{bmatrix}  $
    \item $\mathbf{z}_{k+1} = (\mathbf{x}_{k+1} + \mathbf{u}_k)_+$
    \item $\mathbf{u}_{k+1} = \mathbf{u}_k + \mathbf{x}_{k+1} - \mathbf{z}_{k+1}$
\end{enumerate}
Where $(1)$ represents the KKT conditions for equality-constrained minimization of $f$,  $(2)$ is projection onto the positive orthant, and  $(3)$ is the dual variable update.

\paragraph{Sequential Quadratic Programming.}
For an optimization mapping defined by Problem (\ref{eq:opt_generic}) where $f$, $g$ and $h$ are continuously differentiable, define the operator $\mathcal{T}$ as:
\begin{subequations}
\begin{align}
    \mathcal{T}(\mathbf{x},\pmb{\lambda}) = \argmin_{\mathbf{d}} &\;\;
     \nabla f(\mathbf{x})^T \mathbf{d} + \mathbf{d}^T  \nabla^2 \mathcal{L}(\mathbf{x},\pmb{\lambda}) \mathbf{d} \\
    \texttt{s.t.} \;\; &
    h(\mathbf{x}) + \nabla h(\mathbf{x})^T \mathbf{d} = \mathbf{0}\\
     &g(\mathbf{x}) + \nabla g(\mathbf{x})^T \mathbf{d} \leq \mathbf{0}
\end{align}
\end{subequations}
where dependence of each function on parameters $\mathbf{c}$ is hidden. The function $\mathcal{L}$ is a Lagrangian function of Problem (\ref{eq:opt_generic}).  Then given initial estimates of the primal and dual solution $(x_0,\lambda_0)$, sequential quadratic programming is defined by 
\begin{subequations}
\begin{align}
    (\mathbf{d},\pmb{\mu}) =  \mathcal{T}(\mathbf{x}_{k},\pmb{\lambda}_{k}) \label{line:1}\\
    \mathbf{x}_{k+1} = \mathbf{x}_k + \alpha_k \mathbf{d} \\
    \pmb{\lambda}_{k+1} =   \alpha_k( \pmb{\mu} - \pmb{\lambda}_k) 
\end{align}
\end{subequations}
Here, the inner optimization $\mathcal{O} = \mathcal{T}$ as in Section \ref{sec:unfolding}.

\paragraph{Denoising Problem - Quadratic Programming form}
The following quadratic program is equivalent to the unconstrained denoising problem of Section \ref{sec:Experiments}:
\begin{subequations}
\label{eq:denoiser_QP}
\begin{align}
    \mathbf{x}^{\star}(\mathbf{D}) = \argmin_{\mathbf{x},\mathbf{t}} &\;\;
    \frac{1}{2} \| \mathbf{x}-\mathbf{d} \|^2 + \lambda \overrightarrow{\mathbf{1}} \mathbf{t}    \\
    \textit{s.t.} \;\;
    &    \;\;\;\;\;\;\;\;\;\;\;   \mathbf{D} \mathbf{x} \leq \mathbf{t} \\
    & -\mathbf{t} \leq \mathbf{D} \mathbf{x} 
\end{align}
\end{subequations}

\section{Effect of Stepsize in Fixed-Point Folding}
\label{appendix:stepsize}

Many optimization algorithms rely on a parameter such as a stepsize, which may be constant or change according to some rule at each iteration. Since folded optimization depends on a model of some algorithm at its fixed point to compute gradients, a stepsize must be chosen for its implementation as well. In general, the  stepsize need not be chosen so that the forward pass optimization converges; in all  the example algorithms of this paper, the fixed point remains stationary even for large stepsizes. Instead, the stepsize should be chosen according to its effect on the spectral radius $\rho(\mathbf{\Phi})$ (see Theorem \ref{thm:unfolding_convergence_fixedpt}). For example, in the case of folded PGD, 
\begin{equation}
\label{eq:stepsize-pgd}
        \mathbf{\Phi}=  \frac{\partial}{\partial \mathbf{x}} \mathcal{P}_{\mathbf{C}}( \mathbf{x}^{\star} - \alpha \nabla f (\mathbf{x}^{\star}) ),
\end{equation}
which depends explicitly on the constant stepsize $\alpha$. In practice, it is observed that larger $\alpha$ lead to convergence in less LFPI iterations during fixed-point folding. However when $\alpha$ becomes too large, the resulting gradients explode. For practical purposes, a large range of $\alpha$  result in backward-pass convergence to the same gradients so that careful stepsize selection is not required. For the purpose of optimizing efficiency, $\mathbf{\Phi}$ could be analyzed to determine its optimal $\alpha$, but such an analysis is not pursued within the scope of this paper.

\section{Experimental Details}
\label{appendix:experimental}
Additional details for each experiment of Section \ref{sec:Experiments} are described in their respective subsections below. Note that in all cases, the machine learning models compared in Section \ref{sec:Experiments} use identical settings within each study, with the exception of the optimization components being compared.

\subsection{Nonconvex Bilinear Programming}

\paragraph{Data generation.} Data is generated as follows for the nonconvex bilinear programming experiments. Input data consists of  $1000$ points $\in \mathbb{R}^{10}$ sampled uniformly in the interval $\left[-2, 2\right]$. To produce targets, inputs are fed into a randomly initialized $2$-layer neural network  with $\tanh$ activation, and gone through a nonlinear function $x \cos{2x}+ \frac{5}{2} \log{\frac{x}{x+2}} + x^2\sin{4x}$ to increase the nonlinearity of the mapping between inputs and targets. Train and test sets are split  $90 / 10$. 

\paragraph{Settings.} A 5-layer NN with ReLU activation trained to predict cost $\mathbf{c}$ and $\mathbf{d}$. We train model with Adam optimizer on learning rate of $10^{-2}$ and batch size 32 for 5 epochs.

Nonconvex objective coefficients Q are pre-generated randomly with 15 different seeds. Constraint parameters are chosen arbitrarily as $p=1$ and $q=2$. The average solving time in Gurobi is $0.8333$s, and depends per instance on the predicted parameters $\mathbf{c}$ and $\mathbf{d}$. However the average time tends to be dominated by a minority of samples which take up to $\sim 3$ min. This issue is mitigated by imposing a time limit in solving each instance. While the correct gradient is not guaranteed under early stopping, the overwhelming majority of samples are fully optimized under the time limit, mitigating any adverse effect on training. Differences in training curves under $10$s and $120$s timeouts are negligible due to this effect; the results reported use the $120$s timeout.

\subsection{Enhanced Denoising}

\paragraph{Data generation.} The data generation follows \cite{amos2019optnet}, in which $10000$ random $1D$ signals of length $100$ are generated and treated as targets. Noisy input data is generated by adding random perturbations to each element of each signal, drawn from independent standard-normal distributions. A $90 / 10$ train/test split is applied to the data.

\paragraph{Settings.} A learning rate of $10^{-3}$ and batch size $32$ are used in each training run. Each denoising model is initialized to the classical total variation denoiser by setting the   learned matrix of parameters $\mathbf{D} \in \mathbb{R}^{99 \times 100}$ to the differencing operator, for which $D_{i,i} = 1$ and $D_{i,i+1} = -1 \;\; \forall i$ with all other values $0$.

\subsection{Multilabel Classification}

\paragraph{Dataset.} We follow the  experimental settings and implementation  provided by \cite{Berrada2018SmoothLF}. Each model is evaluated on the noisy top-5 CIFAR100 task. CIFAR-100 labels are organized into 20 “coarse” classes,
each consisting of 5 “fine” labels. With some probability, random noise is added to each label by resampling from the set of “fine”
labels. The $50$k data samples are given a $90 / 10$ training/testing split. 

\paragraph{Settings.}
The DenseNet $40$-$40$ architecture is trained by SGD optimizer with learning rate $10^{-1}$ and batch size $64$ for $30$ epochs to minimize a cross-entropy loss function.

\subsection{Portfolio Optimization}

\paragraph{Data Generation.} The data generation follows exactly the prescription of Appendix D in \cite{elmachtoub2020smart}. Uniform random feature data are mapped through a random nonlinear function to create synthetic price data for training and evaluation. A random matrix is used as a linear mapping, to which nonlinearity is introduced by exponentiation of its elements to a chosen degree. The studies in Section \ref{sec:Experiments} use degrees $1$, $2$ and $3$.

\paragraph{Settings.} A five-layer ReLU network is trained to predict asset prices $\mathbf{c} \in \mathbb{R}^{20}$ using Adam optimizer with learning rate $10^{-2}$ and batch size $32$.

\section{Decision-Focused Learning}
\label{appendix:DFL}
For unfamiliar readers, this section provides background on the decision-focused learning setting, also known as predict-and-optimize, which characterizes the first and last experiments of Section \ref{sec:Experiments} on bilinear programming and portfolio optimization. In this paper, those terms refer to settings in which an optimization mapping 
\begin{subequations} %\tag{OPT}
    \label{eq:opt_dfl}
    \begin{align}
        \mathbf{x}^{\star}(\mathbf{c}) =  \argmin_{\mathbf{x}} &\; f(\mathbf{x},\mathbf{c}) \\        
        \text{subject to: }&\;
        g(\mathbf{x}) \leq \mathbf{0}, \\
        &\; h(\mathbf{x}) = \mathbf{0},
    \end{align}
\end{subequations}
represents a decision model and is parameterized by the vector  $\mathbf{c}$, but only in its objective function. The goal of the supervised learning task is to predict $\hat{\mathbf{c}}$ from feature data such that the resulting $\mathbf{x}^{\star}(\hat{\mathbf{c}})$  optimizes the objective under ground-truth parameters $\bar{\mathbf{c}}$, which is $f(\mathbf{x}^{\star}(\hat{\mathbf{c}}),\bar{\mathbf{c}})$. This is equivalent to minimizing the \emph{regret} loss function:
\begin{equation} %\tag{OPT}
    \label{eq:regret_loss}
        \text{regret}(\hat{\mathbf{c}}, \bar{\mathbf{c}}) = f(\mathbf{x}^{\star}(\hat{\mathbf{c}}),\bar{\mathbf{c}}) - f(\mathbf{x}^{\star}(\bar{\mathbf{c}}),\bar{\mathbf{c}}),
\end{equation}
which measures the suboptimality, under ground-truth objective data, of decisions $\mathbf{x}^{\star}(\hat{\mathbf{c}})$ resulting from prediction $\hat{\mathbf{c}}$. 

When $\mathbf{x}^{\star}$ and $f$ are differentiable, the prediction model for $\hat{\mathbf{c}}$ can be trained to minimize regret directly in an \emph{integrated} predict-and-optimize model.   Since the task amounts to predicting $\hat{\mathbf{c}}$ under ground-truth $\bar{\mathbf{c}}$, a \emph{two-stage} approach is also available which does not require backpropagation through $\mathbf{x}^{\star}$. In the two-stage approach, the loss function $\text{MSE}(\hat{\mathbf{c}},\bar{\mathbf{c}})$ is used to directly target ground-truth parameters, but the final test criteria is still measured by regret. Since the integrated approach minimizes regret directly, it generally outperforms the two-stage in this setting.

\section{Additional Figures}
\label{appendix:Figures}

\begin{figure}
    \includegraphics[width=0.9\linewidth]{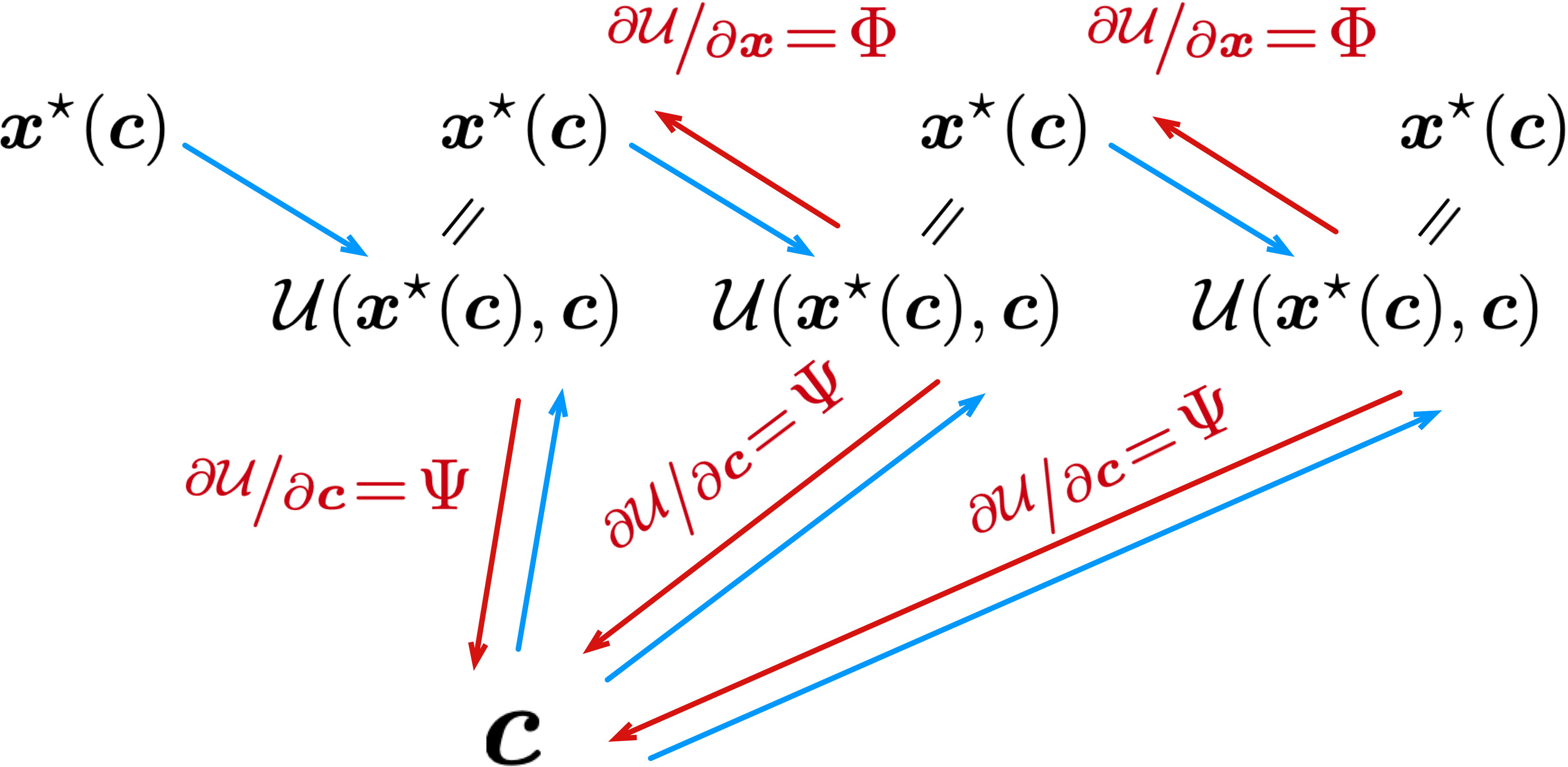}
    \caption{Computational graph for unfolding three iterations of (\ref{eq:opt_iteration}) at a precomputed optimal solution $\mathbf{x}^{\star}$  }
    \label{fig:comp_graph_3}
\end{figure}

\subsection{Enhanced Denoising Experiment}

Figure \ref{fig:denoiser_plots} shows test loss curves, for a variety of $\lambda$, in learning enhanced denoisers with the chosen baseline method $\texttt{qpth}$. As per the original experiment of \cite{amos2019optnet}, the implementation is facilitated by conversion to the quadratic programming form of model (\ref{eq:denoiser_QP}). The results from $\textit{f-FDPG}$ are again shown alongside for comparison. Small differences between the results stem from the slightly different solutions found by their respective solvers at each training iteration, due to their differently-defined error tolerance thresholds.

\begin{figure}
     \centering
     \begin{subfigure}[b]{0.49\linewidth}
         \centering
         \includegraphics[width=\textwidth]{denoiser_dpg_exp.pdf}
         \caption{\textit{f-FDPG}}
         %\label{fig:denoiserdpg}
     \end{subfigure}
     \hfill
     \begin{subfigure}[b]{0.49\linewidth}
         \centering
         \includegraphics[width=\textwidth]{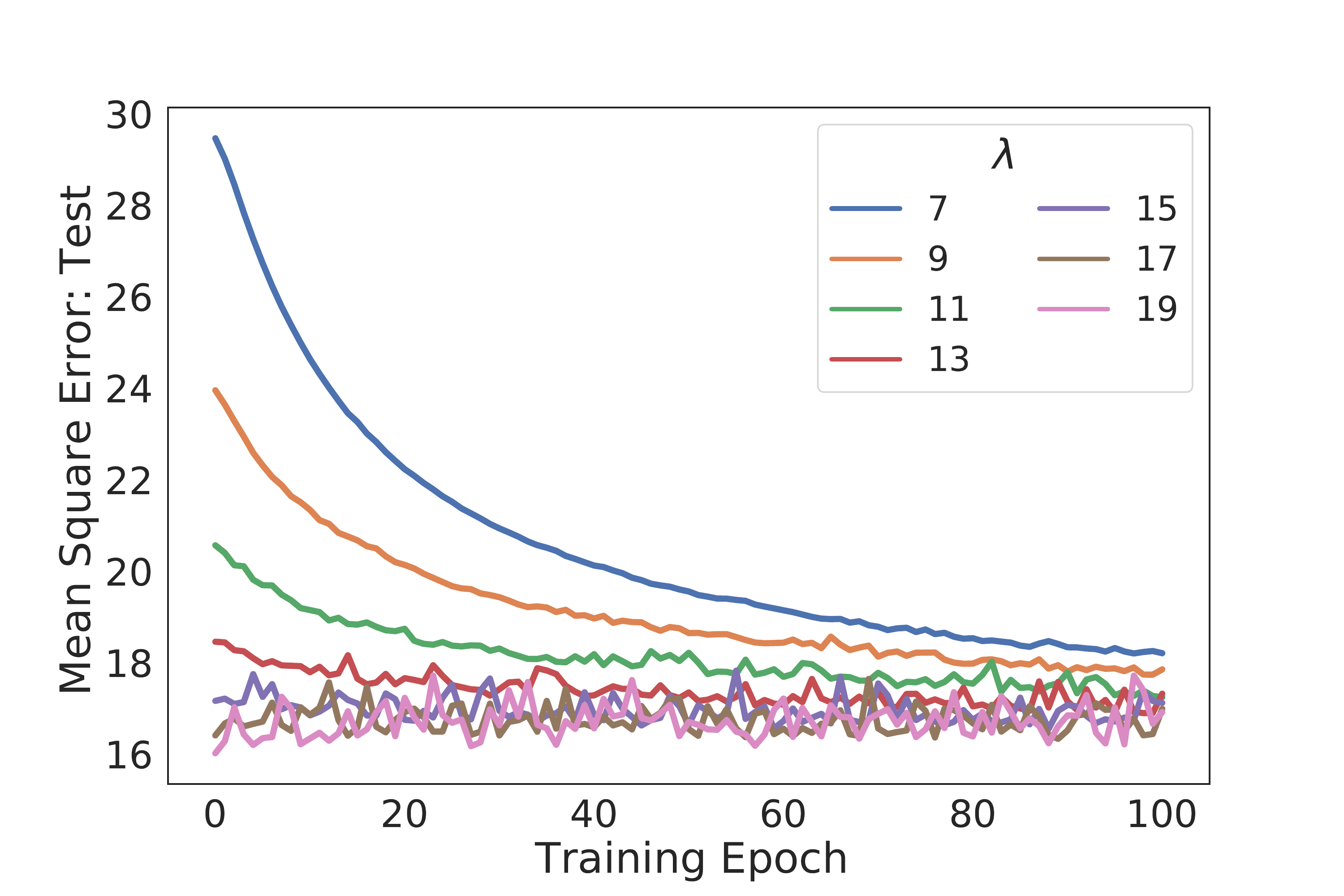}
         \caption{\texttt{qpth}}
         %\label{fig:denoiseroptnet}
     \end{subfigure}
     \hfill

    \caption{Enhanced Denoiser Test Loss}
     \label{fig:denoiser_plots}
\end{figure}

\subsection{Multilabel Classification Experiment}

Figure \ref{fig:ml_plot} shows Top-$1$ and Top-$k$ accuracy on both train and test sets where $k=5$. Accuracy curves are indistinguishable on the training set even after $30$ epochs. On the test set, generalization error manifests slightly differently for each model in the first few epochs.

\begin{figure}
     \centering
     \begin{subfigure}[b]{0.49\linewidth}
         \centering
         \includegraphics[width=\textwidth]{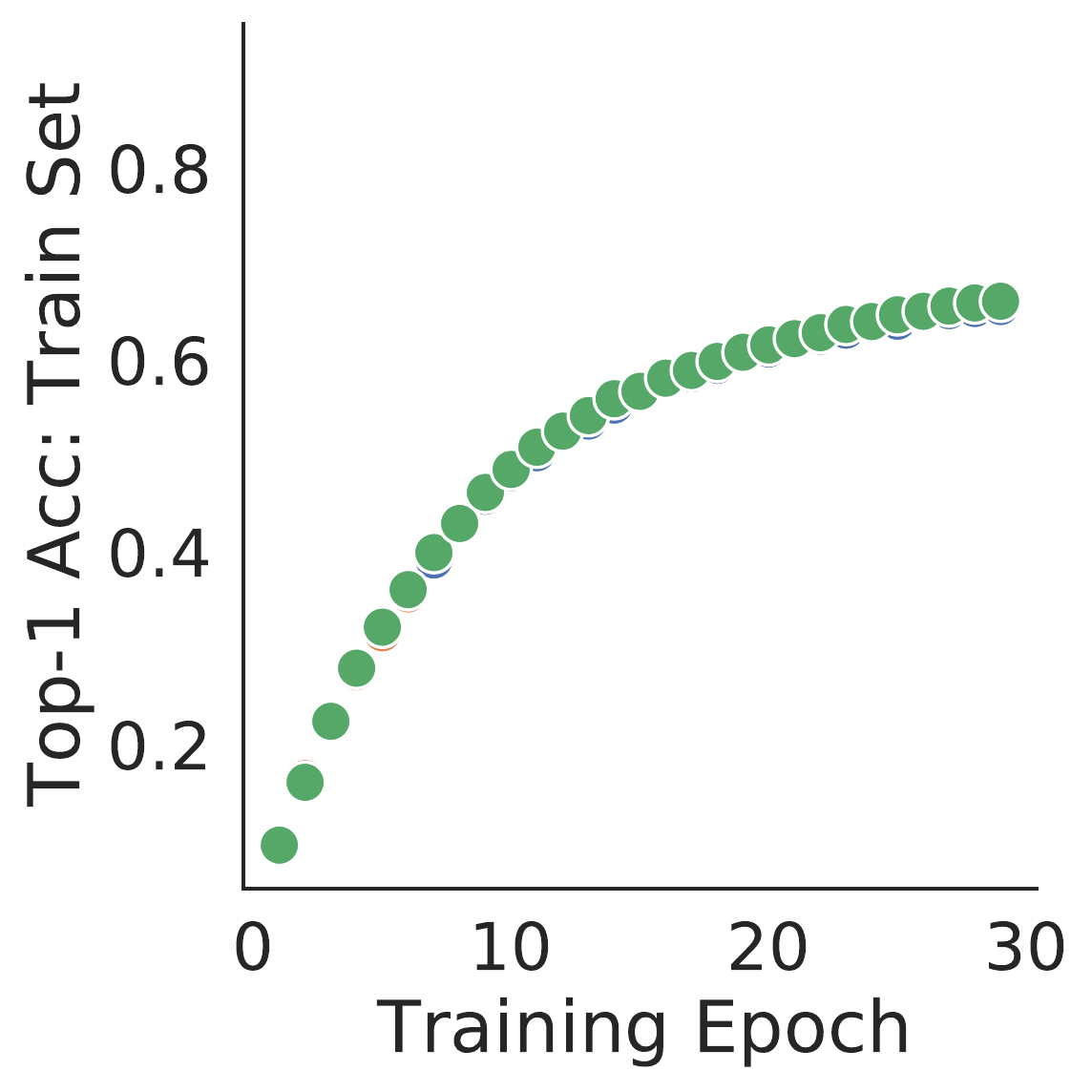}
         \caption{Top-1 Accuracy: Train}
         %\label{fig:denoiserdpg}
     \end{subfigure}
     \hfill
     \begin{subfigure}[b]{0.49\linewidth}
         \centering
         \includegraphics[width=\textwidth]{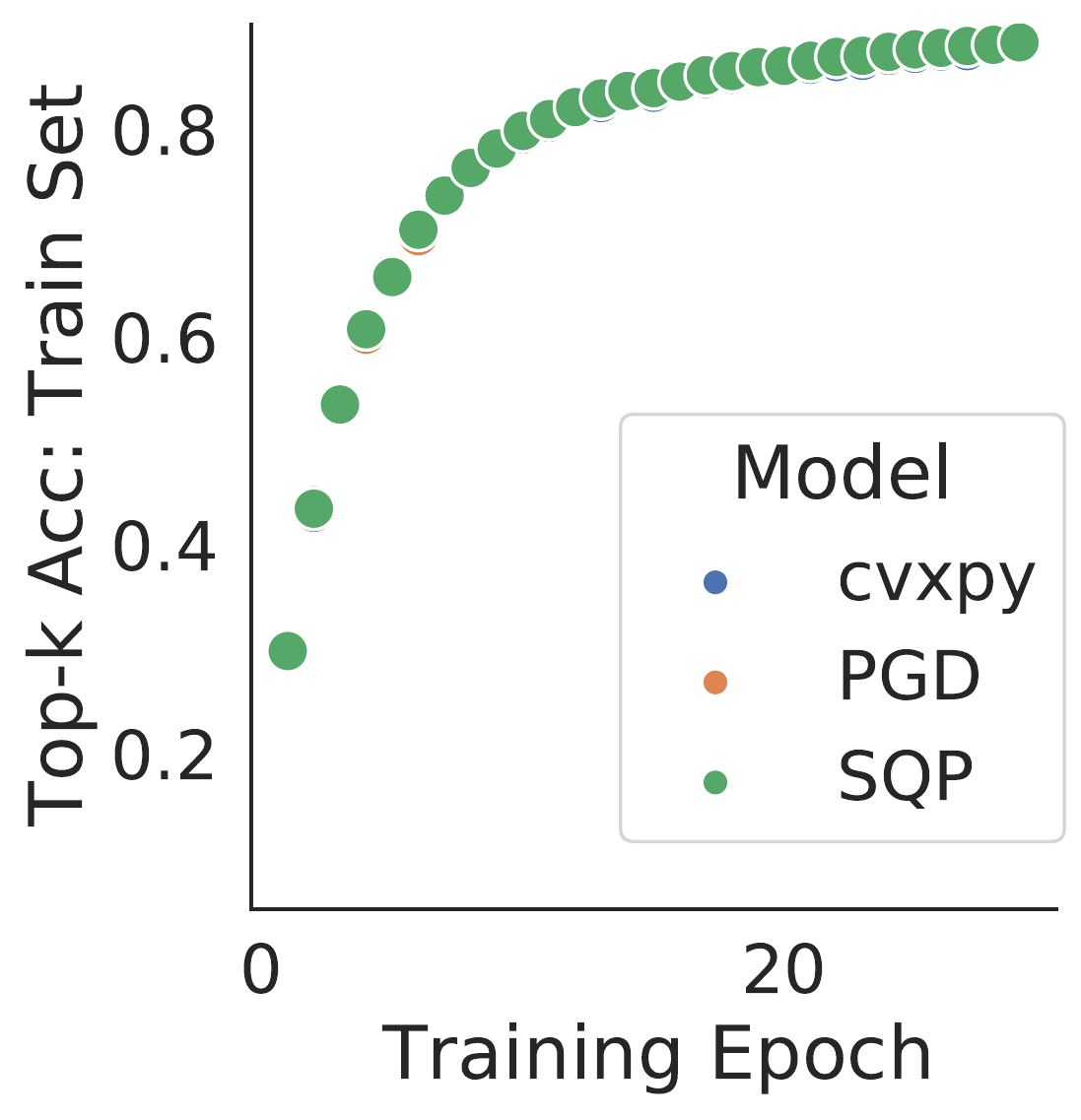}
         \caption{Top-k Accuracy: Train}
         %\label{fig:denoiseroptnet}
     \end{subfigure}
     \hfill
     
         \begin{subfigure}[b]{0.49\linewidth}
         \centering
         \includegraphics[width=\textwidth]{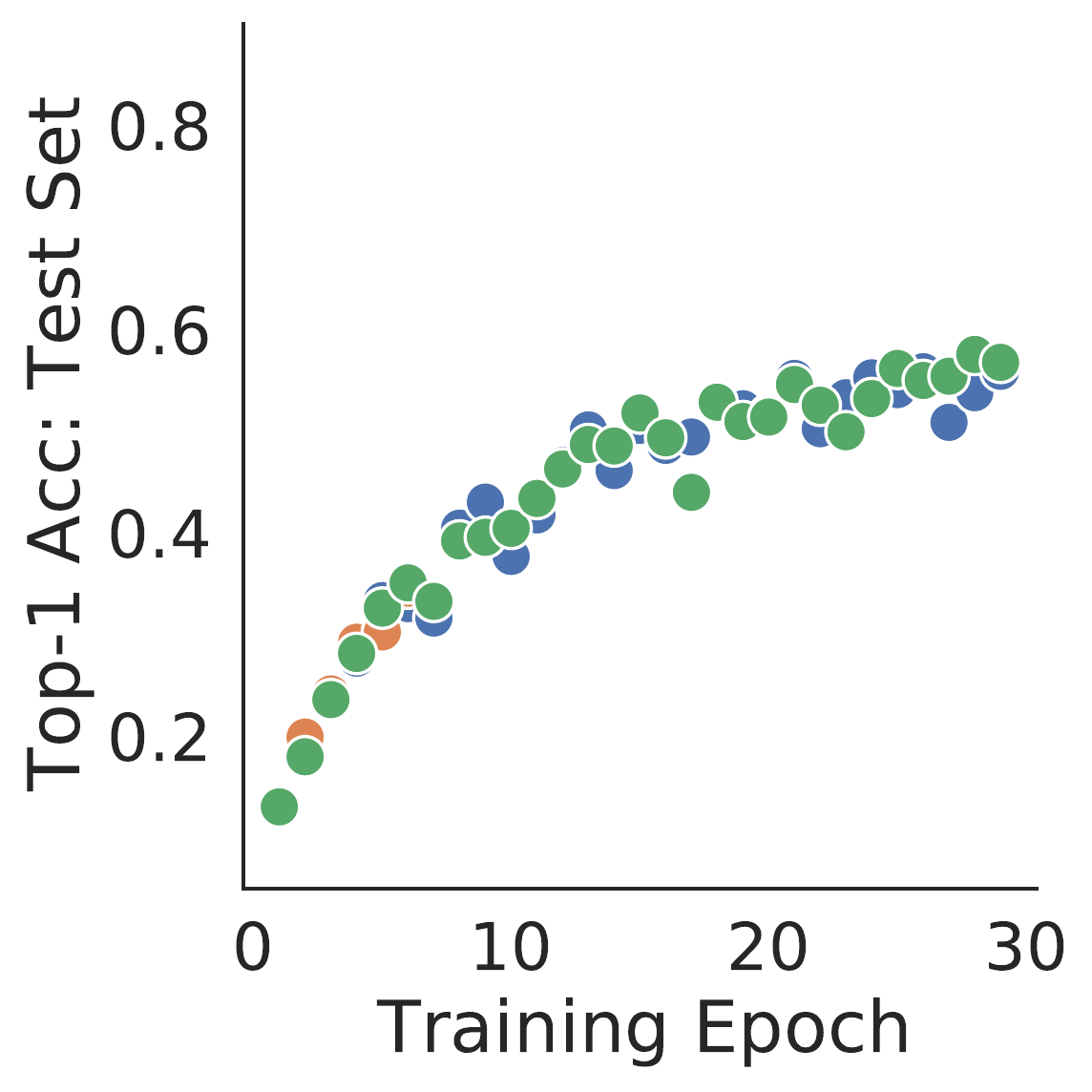}
         \caption{Top-1 Accuracy: Test}
     \end{subfigure}
     \hfill
    \begin{subfigure}[b]{0.49\linewidth}
         \centering
         \includegraphics[width=\textwidth]{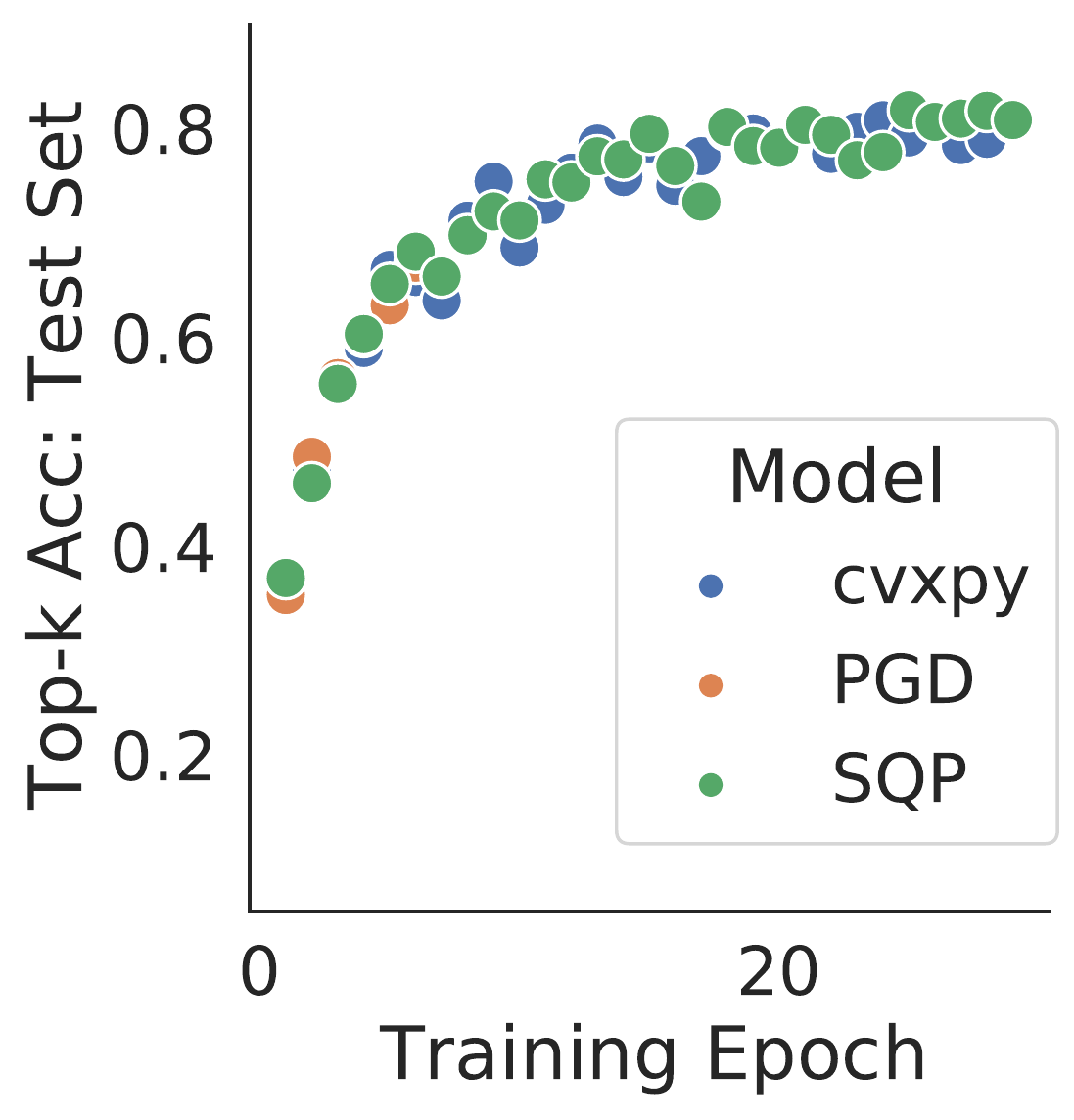}
         \caption{Top-k Accuracy: Train}
        
     \end{subfigure}
     \hfill
    \caption{Multilabel Classification Accuracy}
     \label{fig:ml_plot}
\end{figure}

\subsection{Fixed-Point Unfolding: Computational Graph}
Figure \ref{fig:comp_graph_3} shows a simplified computational graph of unfolding the iteration (\ref{eq:opt_iteration}) at a precomputed fixed point $\mathbf{x}^{\star}$. Forward pass operations are shown in blue arrows, and consist of repeated application of the update function $\mathcal{U}$. Its first input, $\mathbf{x}^{\star}(\mathbf{c})$, is produced the previous call to $\mathcal{U}$ while the second input $\mathbf{c}$ is at the base of the graph.  The corresponding backward passes are shown in red, as viewed through the Jacobians $\frac{\partial \mathcal{U}}{\partial \mathbf{x}}$ and $\frac{\partial \mathcal{U}}{\partial \mathbf{c}}$, which equal $\bf{\Phi}$ and $\bf{\Psi}$ at each iteration since $\mathbf{x}_k = \mathbf{x}^{\star} \;\; \forall k$. This causes the resulting multivariate chain rule to take the linear fixed-point iteration form of Lemma \ref{lemma:linear-iteration}.
\bibliographystyle{named}
\bibliography{bib}

\end{document}